\documentclass[dvipsnames,format=sigconf,anonymous=false,review=false]{acmart}

\usepackage{algorithm}
\usepackage{algorithmic}
\usepackage{utils}

\AtBeginDocument{%
  }

\copyrightyear{2026}
\acmYear{2026}
\setcopyright{cc}
\setcctype{by}
\acmConference[GECCO '26]{Genetic and Evolutionary Computation Conference}{July 13--17, 2026}{San Jose, Costa Rica}
\acmBooktitle{Genetic and Evolutionary Computation Conference (GECCO '26), July 13--17, 2026, San Jose, Costa Rica}
\acmDOI{10.1145/3795095.3805122}
\acmISBN{979-8-4007-2487-9/2026/07}

\begin{document}

\title[Convergence Analysis of Evolution Strategies for Mixed-Integer Optimization]{Convergence Analysis of Evolution Strategies\\for Mixed-Integer Optimization}


\author{Ryoki Hamano}
\email{hamano_ryoki_xa@cyberagent.co.jp}
\orcid{0000-0002-4425-1683}
\affiliation{%
  \institution{CyberAgent}
  \city{Shibuya}
  \state{Tokyo}
  \country{Japan}
  \postcode{150-0042}
}

\author{Kento Uchida}
\email{uchida-kento-fz@ynu.ac.jp}
\orcid{0000-0002-4179-6020}
\affiliation{%
  \institution{Yokohama National University}
  \city{Yokohama}
  \state{Kanagawa}
  \country{Japan}
  \postcode{240-8501}
}

\author{Shinichi Shirakawa}
\email{shirakawa-shinichi-bg@ynu.ac.jp}
\orcid{0000-0002-4659-6108}
\affiliation{%
  \institution{Yokohama National University}
  \city{Yokohama}
  \state{Kanagawa}
  \country{Japan}
  \postcode{240-8501}
}

\renewcommand{\shortauthors}{R. Hamano et al.}

\begin{abstract}
  Mixed-integer extensions of evolution strategies (ES) that discretize selected coordinates of sampled continuous vectors often impose a lower bound on the standard deviation of integer variables to prevent premature convergence.
While these methods show promising empirical results, this handling can slow the convergence of continuous variables, and its impact has lacked a clear theoretical account.
In this paper, we provide a convergence analysis of evolution strategies for mixed-integer optimization, inspired by the drift analysis of the (1+1)-ES in the continuous domain.
Specifically, we consider two (1+1)-ES variants for mixed-integer
domains: (1+1)-LB-ES, which introduces a lower bound on the standard deviation for integer variables, and (1+1)-LUB-ES, which combines both lower and upper bounds to enhance the convergence of the continuous variables.
Focusing on the optimization phase after the integer variables have been optimized, we rigorously analyze their convergence behavior on a benchmark function designed for mixed-integer domains.
Our results show that (1+1)-LB-ES can suffer from premature convergence when the number of integer variables is large, while (1+1)-LUB-ES achieves linear convergence under suitable parameter settings.
These findings provide theoretical insights into the impact of integer handling on convergence performance and guidance for the design of mixed-integer ES.

\end{abstract}

\ccsdesc[500]{Mathematics of computing~Probabilistic algorithms}
\ccsdesc[500]{Theory of computation~Theory of randomized search heuristics}

\keywords{mixed-integer black-box optimization, evolution strategies, linear convergence, premature convergence}


\maketitle

\section{Introduction}

Evolution strategies (ES)~\cite{Rechenberg:1973, Schwefel:1995, hansen_adapting_1996, Back:1996, Beyer-Schwefel:2002, Rudolph:2012, Back:2013, ES-Hansen:2015} represent a powerful approach for continuous black-box optimization.
Specifically, the covariance matrix adaptation evolution strategy (CMA-ES)~\cite{hansen_adapting_1996}, a state-of-the-art variant of ES, is known as a promising continuous black-box optimizer.
These ES-based algorithms have seen successful application in various industrial and AI-related domains~\cite{cmaes-application1,cmaes-application2,salimans2017es,cmaes-application3,cmaes-application4}.
In addition to their strong empirical results, the theoretical foundations of ES are well-established, particularly in comparison to other evolutionary algorithms in the continuous domain.

One powerful approach to analyzing ES is drift analysis in continuous spaces~\cite{akimoto:gecco:2018,morinaga:2019:foga,morinaga:2021:foga,morinaga:2024:tevc}.
In contrast to traditional Markov chain analysis, this method has successfully revealed how the problem-specific parameters, such as the number of dimensions, affect the optimization performance.
Akimoto~et~al.~\cite{akimoto:gecco:2018} investigate (1+1)-ES with the success-based step-size adaptation~\cite{Rechenberg:1973}, and show the expectation of the first hitting time of (1+1)-ES required to find an $\epsilon$-neighbor of the optimum is $\Theta(d \cdot \log( 1 / \epsilon ))$ on the $d$-dimensional sphere function.
This approach is extended to broader classes such as strongly convex and Lipschitz smooth functions~\cite{morinaga:2019:foga, morinaga:2021:foga, morinaga:2024:tevc}.

While ES is widely used in continuous optimization, they have also been extended to integer and mixed-integer domains.
Early work by Rudolph addressed ES on integer-valued search spaces~\cite{Rudolph:1994, Rudolph:2012}, with doubly geometric mutations motivated via maximum-entropy arguments in~\cite{Rudolph:1994}.
Subsequently, ES-based approaches were developed to jointly handle continuous, integer, and nominal (categorical) discrete variables~\cite{Emmerich:2000, Li:2013}.
In these studies, integer and discrete variables were handled through mutation operators designed specifically for those domains.


More recently, mixed-integer extensions of CMA-ES and related ES-based algorithms have been developed in which integer variables are generated by discretizing selected coordinates of sampled continuous vectors~\cite{hansen_cma-es_2011, CMAwM:2022, 1+1-CMA-ESwM:2023, MI-NES:2023, CMAwM-TELO:2024, LB+IC-CMA-ES:2024}.
This allows mixed-integer problems to be addressed while largely preserving the framework of a promising continuous optimizer.
However, if the standard deviation in the integer dimensions becomes too small relative to the integer intervals, the generated integers can be fixed to a single value.
To address this issue, these methods impose a lower bound on the standard deviation in these dimensions.
While these approaches have shown promising experimental results, the lower bound has a substantial impact on performance and requires careful tuning.
If this bound is too small, integer variables prematurely converge to suboptimal values, leading to the fixation of integers.
Conversely, if it is too large, even when integer variables reach optimal values, there remains a high probability of generating nearby integers, which can impair the efficient convergence of continuous variables.
Therefore, recommended lower-bound values have been discussed based on experimental and analytical evaluations~\cite{CMAwM:2022, LB+IC-CMA-ES:2024, CatCMAwM:2025}.
However, theoretical guarantees for this integer-handling approach have not been developed, and its impact on the convergence of continuous variables remains insufficiently understood.

In this study, we analyze evolution strategies with integer handling for mixed-integer optimization. We focus on the variants of (1+1)-ES with integer handling to theoretically investigate how integer handling affects the ES performance in a mixed-integer problem.
In particular, inspired by~\cite{LB+IC-CMA-ES:2024}, we consider two variants of (1+1)-ES for mixed-integer optimization: (1+1)-ES only with a lower bound $\sigLB$ of the standard deviation for the integer variable, termed (1+1)-LB-ES, and (1+1)-ES with both the lower bound $\sigLB$ and the upper bound introduced in \cite{CatCMAwM:2025}, termed (1+1)-LUB-ES.
To isolate the effect of integer handling on the convergence of continuous variables, we focus on the phase where the optimization of integer variables is complete.
To ease the difficulty of analysis, we develop \textsc{LexicoSphereInt}, on which the behavior of the evolution strategies is similar to that of the well-used benchmark function, \textsc{SphereInt}.
Specifically, on \textsc{LexicoSphereInt}, the ranking of two solutions $(\x,\z)$ and $(\x',\z')$, where $\x,\x' \in \mathbb{R}^{\dco}$ and $\z,\z' \in \mathbb{Z}^{\din}$, is determined by comparing the evaluation values $\|\z\|^2$ and $\|\z'\|^2$ for integer variables for the first time, then if they are equal, the evaluation values $\|\x\|^2$ and $\|\x'\|^2$ continuous variables are compared.

We analyze the behavior of (1+1)-LB-ES and (1+1)-LUB-ES starting with the optimal integer variable $\mm[0] = \mathbf{0}$ to investigate how the additional integer-handlings affect the optimization performance of the continuous variable $\m[t]$.
By applying the additive drift theorem, we show that the step-size $\sig$ of (1+1)-LB-ES decreases faster than the distance $\| \m \|$ between the mean vector and the optimal solution when the number $\din$ of integer variables is significantly large.
In addition, we theoretically show the premature convergence of (1+1)-LB-ES, i.e., there exists a constant value $C_\mathrm{lower} > 0$ that holds $\E\left[ \| \m[s \tau] \| \right] \geq C_\mathrm{lower} \cdot \| \m[0] \|$ for any $\tau \in \mathbb{N}$.
Moreover, in the analysis of (1+1)-LUB-ES, we prove that the (1+1)-LUB-ES converges linearly, i.e., the expected first hitting time of (1+1)-LUB-ES required to find an $\epsilon$-neighbor of the optimum is $\Theta(\dco \cdot \log( 1 / \epsilon ))$.
We also investigate the condition for the target success rate $1/s$ in the success-based step-size adaptation to realize the linear convergence depending on the setting of the lower bound $\sigLB$.
Numerical simulation reveals that the premature convergence of (1+1)-LB-ES occurs even with a moderate number of dimensions, while (1+1)-LUB-ES achieves a linear convergence rate with a suitable setting of $s$ satisfying the derived condition.




\section{Formulations}

\subsection{Notation}
We denote the indicator function for condition $C$ as $\mathbf{1}_{\{ C \}}$.
The $d$-dimensional identity matrix is denoted  $\mathbf{I}_d$.
A multivariate normal distribution is denoted $\mathcal{N}$.
The cumulative distribution function of the standard normal distribution and its inverse function are denoted as $\Phi$ and $\Phi^{-1}$, respectively.
The $i$-th element of a vector $\bv$ and $i$-th diagonal element of a matrix $\boldsymbol{A}$ are denoted $[\bv]_i$ and $\langle \boldsymbol{A} \rangle_i$, respectively.
The discretizing function $\Int [\cdot] : \R^{\din} \to \Z^{\din}$ is defined as follows. 
\begin{align}
    \left[ \Int [\bv] \right]_i = \left\lfloor [\bv]_i + \frac12 \right\rfloor \quad \text{for} \enspace i = 1, \ldots, \din
\end{align}
The symbols ${O}$, ${o}$, $\Theta$, and $\Omega$ represent the Landau notations. We denote $g(x) \in - {O}( h(x) )$ when there exist $x_0 > 0$ and $M > 0$ such that $g(x) \geq - M \cdot | h(x) |$ for all $x > x_0$.

\subsection{Algorithms}
We consider a simple evolution strategy for mixed-integer spaces, (1+1)-LB-ES with $1/s$-success rule, defined in Algorithm~\ref{alg:mixint-ES}.
The (1+1)-LB-ES aims to minimize the objective function $f:\R^{\dco} \times \Z^{\din} \to \R$ by mutating the elitist solution $(\m,\mm) \in \R^{\dco} \times \Z^{\din}$ to generate the next candidate solution $(\x_t, \z_t) \in \R^{\dco} \times \Z^{\din}$ as
\begin{align}
    \x_t = \m + \sig \boldsymbol{\xi}^{\mathrm{co}}_t \quad \text{and} \quad \z_t = \Int \left[ \mm + \sig \D \boldsymbol{\xi}^{\mathrm{in}}_t \right] \enspace,
\end{align}
where $\bxi_t^{\mathrm{co}} \sim \mcN(\boldsymbol{0}, \mathbf{I}_{\dco})$ and $\bxi_t^{\mathrm{in}} \sim \mcN(\boldsymbol{0}, \mathbf{I}_{\din})$.
\rev{The diagonal matrix $\D$ controls the coordinate-wise mutation scale in the integer block.}
For the sake of analysis, we introduce the notation $\preccurlyeq_f$ to denote the ranking of solutions on $f$, where $(\x, \z) \preccurlyeq_f (\x', \z')$ indicates $(\x', \z')$ is no worse than $(\x,\z)$ for given mixed continuous-integer solutions $(\x, \z), (\x', \z') \in \R^{\dco} \times \Z^{\din}$ (further detailed in Section~\ref{ssec:problems}).
This algorithm is based on (1+1)-ES, with discretization of continuous vectors and simple integer handling proposed in \cite{LB+IC-CMA-ES:2024}.
It can also be interpreted as a simplified version of (1+1)-CMA-ES with Margin ((1+1)-CMA-ESwM)~\cite{1+1-CMA-ESwM:2023}, which uses the essentially same integer handling as in \cite{LB+IC-CMA-ES:2024}\footnote{The (1+1)-CMA-ESwM employs a different approach when handling minimum and maximum integers in a search space consisting of finite sets of integers.}.

The parameter $\sigLB$ is the lower bound of the standard deviation for integer variables $\sig[t]\langle \D[t] \rangle_i$.
The larger $\sigLB$ is, the greater the probability $\Pr([\mm[t]]_i \neq [\z_t]_i)$ becomes, preventing mutated integers from being fixed to a single value.
This also implies that when $\mm[t]$ is optimal, a larger $\sigLB$ increases the likelihood of mutations to non-optimal integers, which has a negative effect on the convergence of continuous variables.
Therefore, the parameter $\sigLB$ requires careful adjustment considering the trade-off between handling of integer fixation and convergence performance.
In the (1+1)-CMA-ESwM, the parameter corresponding to $\sigLB$, referred to as \emph{margin} in \cite{1+1-CMA-ESwM:2023}, is recommended to be set such that $2\Phi\left(- \frac{1}{2\sigLB} \right) = \frac{1}{\dco + \din}$, which ensures $\Pr([\mm[t]]_i \neq [\z_t]_i) \geq \frac{1}{\dco + \din}$.

\citet{LB+IC-CMA-ES:2024} find the importance of {\it the successful integer mutation}, which is the mutation of the integer variables where the mutated solution is regarded as a superior solution.
In addition, \citet{CatCMAwM:2025} introduce an upper bound, in addition to the lower bound, on the standard deviation in the dimensions without the successful integer mutation.
Their study experimentally showed that this upper bound enhances the efficiency of convergence for continuous variables.
In this study, to investigate the effectiveness of such an approach, by incorporating the upper bound into the (1+1)-LB-ES for integers without the successful integer mutation as 
\begin{align}
    \sig[t+1] \langle \D[t+1] \rangle_i \leftarrow \max \left\{ \sigLB, \min \left\{ \sig[t+1] \langle \D[t] \rangle_i, \sig[t] \langle \D[t] \rangle_i \right\} \right\} \,,
\end{align} 
we develop (1+1)-LUB-ES, defined in Algorithm~\ref{alg:mixint-ES-LUB}.
The upper bound is motivated by preventing the mutation rate of integers without the successful integer mutation from increasing.

In this study, we are interested in understanding how integer handling affects the convergence of continuous variables. To theoretically analyze this, we consider the process after the integer part of the elitist solution reaches the optimal value, i.e., $\mm[t] = \boldsymbol{0}$\rev{, after shifting the coordinate system without loss of generality.} It should be noted that these algorithms cannot utilize information about this optimality.

\begin{algorithm}[t]
    \caption{(1+1)-LB-ES with $1/s$-success rule}
    \begin{algorithmic}[1]
        \STATE \textbf{Input}: $\m[0] \in \R^{\dco}$, $\mm[0] \in \Z^{\din}$, $\sig[0] > 0$, $\D[0] \in \R^{\din \times \din}$
        \STATE \textbf{Parameter:} $\alpha > 1$, $s > 1$, $\sigLB \geq 0$
        \FOR{$t = 0, 1, \ldots, $ until the stopping criterion is met}
            \STATE sample $\bxi_t^{\mathrm{co}} \sim \mcN(\boldsymbol{0}, \mathbf{I}_{\dco})$ and $\bxi_t^{\mathrm{in}} \sim \mcN(\boldsymbol{0}, \mathbf{I}_{\din})$
            \STATE $\x_t \leftarrow \m[t] + \sig[t] \bxi_t^{\mathrm{co}}$
            \STATE $\z_t \leftarrow \Int\left[\mm[t] + \sig[t] \D[t] \bxi_t^{\mathrm{in}}\right]$
            \IF{$(\m[t], \mm[t]) \preccurlyeq_f (\x_t, \z_t)$}
                \STATE $(\m[t+1],\mm[t+1]) \leftarrow (\x_t,\z_t)$
                \STATE $\sig[t+1] \leftarrow \alpha \sig[t]$
            \ELSE
                \STATE $(\m[t+1],\mm[t+1]) \leftarrow (\m[t], \mm[t])$
                \STATE $\sig[t+1] \leftarrow \alpha^{-\frac{1}{s-1}} \sig[t]$
            \ENDIF
            \FOR{$i=1,\ldots,\din$}
                \STATE $\langle \D[t+1] \rangle_i \leftarrow \max \left\{ \sigLB, \sig[t+1] \langle \D[t] \rangle_i \right\} / \sig[t+1]$
            \ENDFOR
        \ENDFOR
    \end{algorithmic}
    \label{alg:mixint-ES}
\end{algorithm}

\begin{algorithm}[t]
    \caption{(1+1)-LUB-ES with $1/s$-success rule}
    \begin{algorithmic}[1]
        \STATE \textbf{Input:} $\m[0] \in \R^{\dco}$, $\mm[0] \in \Z^{\din}$, $\sig[0] > 0$, $\D[0] \in \R^{\din \times \din}$
        \STATE \textbf{Parameter:} $\alpha > 1$, $s > 1$, $\sigLB \geq 0$
        \FOR{$t = 0, 1, \ldots, $ until the stopping criterion is met}
            \STATE sample $\bxi_t^{\mathrm{co}} \sim \mcN(\boldsymbol{0}, \mathbf{I}_{\dco})$ and $\bxi_t^{\mathrm{in}} \sim \mcN(\boldsymbol{0}, \mathbf{I}_{\din})$
            \STATE $\x_t \leftarrow \m[t] + \sig[t] \bxi_t^{\mathrm{co}}$
            \STATE $\z_t \leftarrow \Int\left[\mm[t] + \sig[t] \D[t] \bxi_t^{\mathrm{in}}\right]$
            \IF{$(\m[t], \mm[t]) \preccurlyeq_f (\x_t, \z_t)$}
                \STATE $(\m[t+1],\mm[t+1]) \leftarrow (\x_t,\z_t)$
                \STATE $\sig[t+1] \leftarrow \alpha \sig[t]$
                \FOR{$i=1,\ldots,\din$}
                    \IF{$[\mm[t]]_i \neq [\z_t]_i$}
                        \STATE $\langle \D[t+1] \rangle_i \leftarrow \max \left\{ \sigLB, \sig[t+1] \langle \D[t] \rangle_i \right\} / \sig[t+1]$
                    \ELSE
                        \STATE 
                        $\langle \D[t+1] \rangle_i \leftarrow 
                        \max \left\{
                            \sigLB,\right.$\\
                        \hspace*{5em} $\left.\min \left\{ \sig[t+1]\langle \D[t] \rangle_i,~ \sig[t]\langle \D[t] \rangle_i \right\}
                        \right\} /\sig[t+1]$ \label{algo:line:upper-1}
                    \ENDIF
                \ENDFOR
            \ELSE
                \STATE $(\m[t+1],\mm[t+1]) \leftarrow (\m[t],\mm[t])$
                \STATE $\sig[t+1] \leftarrow \alpha^{-\frac{1}{s-1}} \sig[t]$
                \FOR{$i=1,\ldots,\din$}
                    \STATE 
                        $\langle \D[t+1] \rangle_i \leftarrow 
                        \max \left\{
                            \sigLB,\right.$\\
                        \hspace*{5em} $\left.\min \left\{ \sig[t+1]\langle \D[t] \rangle_i,~ \sig[t]\langle \D[t] \rangle_i \right\}
                        \right\} /\sig[t+1]$ \label{algo:line:upper-2}
                \ENDFOR
            \ENDIF
        \ENDFOR
    \end{algorithmic}
    \label{alg:mixint-ES-LUB}
\end{algorithm}

\subsection{Problems} \label{ssec:problems}
Recent studies in mixed-integer black-box optimization have targeted the minimization of continuous benchmark functions where certain variables are discretized~\cite{hansen_cma-es_2011, tusar_mixed-integer_2019, CMAwM:2022, 1+1-CMA-ESwM:2023, MI-NES:2023, CMAwM-TELO:2024, LB+IC-CMA-ES:2024, CatCMAwM:2025}.
This study first focuses on \textsc{SphereInt}, defined as $f(\x, \z) = \| \x \|^2 + \| \z \|^2$, which is a mixed-integer variant of the well-known continuous sphere function.
Additionally, we will conduct theoretical analyses focusing on the phase where the integer part has been optimized.
However, there is a technical difficulty in theoretical analysis even for such a simple function with such a limited setting.
After finding the optimal integer values, the integer part of the elitist solution is usually fixed to the optimal integers in actual algorithms, while theoretical analysis should account for the rare event where the integer part is updated to non-optimal integers.
This makes it difficult to analyze the effect of integer handling on the convergence of continuous variables.

To address this difficulty, we consider \obj{}, where the ranking of two solutions $(\x, \z)$ and $(\x', \z')$ is defined as follows.
\begin{align}
    &(\x, \z) \preccurlyeq_f (\x', \z') \notag \\
    &\enspace \Leftrightarrow \| \z \|^2 > \| \z' \|^2 \lor \Bigl[ \| \z \|^2 = \| \z' \|^2 \land \| \x \|^2 \geq \| \x' \|^2 \Bigr]
\end{align}
When optimizing \obj{} with the (1+1)-LB-ES or the (1+1)-LUB-ES, once the integer part of the elitist solution reaches its optimal value, it remains optimal after that point.
This setting does not substantially deviate from practical settings, as empirical results show that the discrete variables are typically optimized much earlier than the continuous ones in mixed-variable evolutionary algorithms~\cite{Timo:2025}.


When optimizing \textsc{LexicoSphereInt} using the (1+1)-LB-ES or the (1+1)-LUB-ES with $\mm[0] = \boldsymbol{0}$, $\langle \D[t] \rangle_i$ takes an identical value for $i = 1, \ldots, \din$.
Therefore, to simplify the notation, we denote $\langle \D[t] \rangle_1 = \cdots = \langle \D[t] \rangle_{\din} = \langle \D[t] \rangle$.

\section{Analysis}

In this section, we consider the optimization process after the integer part of the elitist solution reaches the optimal value, i.e., $\mm = \mathbf{0}$, and investigate the dynamics of the step-size $\sig[t]$ and the continuous part of the elitist solution.
The setting of $\sigLB$ significantly affects the dynamics, which is often determined by $\dco$ and $\din$.
We investigate the following setting, parameterized by $\pinmut$:
\begin{align}
    \label{eq:sigLB}
    \sigLB = - \left( 2 \Phi^{-1}\left( \new{ \frac{\pinmut}{2} } \right) \right)^{-1} > 0
\end{align}
This lower bound maintains the generation probability of other integers above $\new{\pinmut}$, namely, $\Pr( \z \neq \mm[] \mid \z \sim \mathcal{N}(\mm[], \sig[]^2 \D[]^2 ) ) \geq \new{\pinmut}$.
A commonly used choice in (1+1)-CMA-ESwM is obtained by setting $\pinmut = 1/(\dco + \din)$.
We assume that $\dco$ and $\din$ are independently given and $\pinmut \in \Theta(1/\din)$ with respect to $\din$.
In addition, we also assume that $s$ is a constant independent of the dimension, and the initial step-size is given by $\sigma_0 = (\sigLB)^{\frac{K}{s-1}}$ for some $K \in \mathbb{Z}$ to ease the analysis.
We denote the natural filtration as $\{ \mathcal{F}_t : t \in \mathbb{N}_0 \}$.
Then, we investigate the dynamics of (1+1)-LB-ES and (1+1)-LUB-ES on \obj{} starting with the initial solution satisfying $\mm[0] = \mathbf{0}$.
The proofs of propositions and main theorems are provided in
Appendices~\ref{apdx:proofs-propositions} and~\ref{apdx:proofs-main},
respectively.

\subsection{Useful Properties of the Normal Distribution}
\label{sec:properties:gaussian}

As the first step, we introduce several propositions related to normal distribution useful in our analysis. 
The first proposition shows a lower bound of the ratio $\Phi(-a) / \Phi(-b)$ of the cumulative distribution functions for $0 < a < b$.
\footnote{Proposition~\ref{proposition:cdf-ratio} is also found in~\url{https://mathoverflow.net/questions/270078/}.}

\begin{proposition} \label{proposition:cdf-ratio}
    \label{proposition:gaussian:cdf-ratio}
    Let $\theta > 0$, $\beta > 1$, and $\Phi$ be the cumulative distribution function of the standard normal distribution. Then we have
    \begin{align}
        \Phi\left(- \frac{\theta}{\beta} \right) > \Phi(-\theta) \cdot \exp\left( \frac{\theta^2}{2} \left(1 - \frac{1}{\beta^2} \right) \right) \enspace.
    \end{align}
\end{proposition}

Proposition~\ref{proposition:cdf-ratio} is used to bound the generation probability for optimal integer variable, i.e., $\Pr( \z_t = \boldsymbol{0} \mid \mathcal{F}_t)$, which equals to
\begin{multline}
    \prod_{i=1}^{\din} \Pr\left( -0.5 \leq \left[\sig[t] \D[t] \bxi^{\mathrm{in}}_t \right]_i < 0.5 \mid \mathcal{F}_t \right) \\
    = \left( 1 - 2 \Phi \left( - \frac{1}{2 \sig[t] \langle \D[t] \rangle } \right) \right)^{\din} \!.
\end{multline}
As the step-size increases or decreases with constant factors $\alpha$ and $\alpha^{- \frac{1}{s-1}}$ in a single update, Proposition~\ref{proposition:cdf-ratio} is useful to bound $\Pr( \z_t = \boldsymbol{0} \mid \mathcal{F}_t)$ by considering two cases: $\sig \langle \D \rangle = \sigLB$ and $\sig \langle \D \rangle \geq \sigLB \cdot \alpha^{\frac{1}{s-1}}$.
In the latter case, it will be shown that $\Pr( \z_t = \boldsymbol{0} \mid \mathcal{F}_t)$ becomes significantly small.

The next proposition bounds the inverse function $\Phi^{-1}$ of the cumulative distribution function of the standard normal distribution.

\begin{proposition} \label{proposition:cdf-bound}
    Let $\Phi^{-1}$ be the inverse of the cumulative distribution function of the standard normal distribution. Then, for $0 < y < \Phi(-1)$, we have
    \begin{align}
        \exp\left( (\Phi^{-1}(y))^2 \right) \geq \left( \sqrt{\frac{8\pi}{e}} \cdot y \right)^{- 1} \enspace.
    \end{align}
\end{proposition}

As the setting of $\sigLB$ in~\eqref{eq:sigLB} contains $\Phi^{-1}$, we use Proposition~\ref{proposition:cdf-bound} to bound the generation probability $\Pr( \z_t \neq \mm[t] \mid \mathcal{F}_t)$ of non-optimal integer variable  when $\sig[t] \langle \D[t] \rangle = \sigLB$.

The last proposition shows the expected improvement in (1+1)-ES becomes significantly small when the ratio $\| \m[] \| / \sig[]$ between the distance of the mean vector and the optimal solution and the step-size is large.

\begin{proposition}
    \label{proposition:chi-squared:inverse:trunc}
    Let $\x \sim \mathcal{N}(\m[], \sig[]^2 \cdot \mathbf{I}_d)$ with $d \geq 4$. Then, 
    \begin{align}
        \mathbb{E}\left[ \log \left( \frac{\min\{ \| \x \|, \|\m[]\| \}}{\| \m[] \|} \right) \right] \in - O\left(\left( \frac{ \| \m[] \| }{ \sig[]} \right) ^{- 3} \right) \enspace.
    \end{align}
\end{proposition}

\subsection{Premature Convergence of (1+1)-LB-ES}
\label{sec:lb-es}

In the process of the premature convergence of (1+1)-LB-ES, the decrease rate of the step-size $\sig$ is kept larger than the decrease rate of the distance $\| \m \|$ between the continuous part of the elitist solution and the optimal solution.
In our analysis, we consider the dynamics of the step-size and the mean vector every \rev{block of $s$ iterations} and compare the expected decrease rates of them in logarithm scale.

First, we will show that the expected decrease rates of the step-size in logarithm scale are bounded by constant order with respect to $\din$ from above.
This bound is obtained as follows.
When the elitist solution is updated, the step-size $\sig[]$ is increased, whereas $\langle \D[] \rangle$ remains unchanged. Then, the standard deviation $\sig[] \langle \D[] \rangle$ of the search distribution corresponding to the integer variable becomes larger than $\sigLB \cdot \alpha^{\frac{1}{s-1}}$ for next $(s-1)$ iterations.
Because this large standard deviation makes the generation probability of the optimal integer variables significantly smaller, the step-size is expected to continue decreasing until it holds $\sig[] = \sigLB$.
The next proposition shows the expected decrease rate for the step-size.

\begin{proposition}
    \label{proposition:drift:sig}
    Suppose $\pinmut \in \Theta(1/\din)$, $\mm[0] = \mathbf{0}$, and $\sigma_0 = (\sigLB)^{\frac{K}{s-1}}$ for some $K \in \mathbb{Z}$. \new{Suppose $\sigLB$ is given by~\eqref{eq:sigLB}.} Then, it holds
    \begin{multline}
        \mathbb{E}\left[ \log \left(\frac{ \sig[t+s] }{ \sig[t] }\right) \mid \mathcal{F}_t \right] 
        \leq \left( s \cdot O \left( \din^{- \gamma} \right) - \frac{s}{s - 1} \cdot \frac{1}{2^s} \right) \log (\alpha) \enspace,
        \label{eq:drift:bound:sig}
    \end{multline}
    where $\gamma = \frac{1}{2} \left(1 - \alpha^{- \frac{2}{s-1}}\right) > 0$. 
\end{proposition}

For the continuous part of the elitist solution, the lower bound of the expected decrease rate is obtained based on~\cite[Lemma~4.4]{akimoto:gecco:2018}, which shows the lower bound of the expected improvement of (1+1)-ES on the sphere function.

\begin{proposition}
    \label{proposition:drift:m}
    For $\dco \geq 2$ and $i \in \mathbb{N}$, it holds
    \begin{align}
        \mathbb{E}\left[ \log \left( \frac{ \| \m[t+i] \| }{ \| \m[t] \| } \right) \mid \mathcal{F}_t \right] \geq - \frac{i}{\dco} \enspace.
        \label{eq:drift:bound:m}
    \end{align}
\end{proposition}

Combining Propositions~\ref{proposition:drift:sig} and~\ref{proposition:drift:m}, we show the drift for the logarithm of the ratio between $\| \m[t] \|$ and $\sig[t]$ is a positive constant order with respect to $\din$ at least.
\begin{proposition}
    \label{proposition:drift:ratio}
    Suppose $\pinmut \in \Theta(1/\din)$, $\mm[0] = \mathbf{0}$, and $\sigma_0 = (\sigLB)^{\frac{K}{s-1}}$ for some $K \in \mathbb{Z}$.
    \new{Suppose $\sigLB$ is given by~\eqref{eq:sigLB}.}
    Then, for $\dco \geq 2$, it holds
    \begin{multline}
        \E\left[ \log\left( \frac{\| \m[t+s] \|}{\sig[t+s]} \right) - \log\left( \frac{\| \m[t] \|}{\sig[t]} \right)  \mid \mathcal{F}_t \right] \\
        \geq - \left( s \cdot O \left( \din^{- \gamma} \right) - \frac{s}{s - 1} \cdot \frac{1}{2^s} \right) \log \alpha - \frac{s}{\dco} \enspace.
    \end{multline}
    Additionally, there exist $D \in \mathbb{N}$ and $\epsilon > 0$ such that, for all $\din \geq D$ and 
    \begin{align}
        \label{eq:cond:dco}
        \dco > D_1
        \enspace \text{with} \enspace
        D_1 = \frac{1}{\log \alpha} \left( \frac{1}{s - 1} \cdot \frac{1}{2^s} - \epsilon \right)^{-1} \!,
    \end{align}
    it holds 
    \begin{multline}
        \E\left[ \log\left( \frac{\| \m[t+s] \|}{\sig[t+s]} \right) - \log\left( \frac{\| \m[t] \|}{\sig[t]} \right)  \mid \mathcal{F}_t \right] 
        \geq \frac{s}{D_1} - \frac{s}{\dco} > 0 \enspace.
    \end{multline}
\end{proposition}

According to Proposition~\ref{proposition:drift:ratio}, we can derive the lower bound of the expected ratio $\E [ \| \m[s \tau] \| / \sig[s \tau] ] \in \Omega ( \exp( \tau ) )$. 
In addition, using the additive drift theorem~\cite{AdditiveDtift:ECJ:2003,AllDrift:arxiv:2024} (stated as Theorem~\ref{theorem:additive-drift} in Appendix~\ref{apdx:existing}), we will show $ \| \m[s \tau] \| / \sig[s \tau] \in \Omega ( \exp( \tau ) )$ with probability 1. 
The next proposition shows that there exists a constant upper bound of the variance of the logarithm of the ratio, which is one of the assumptions of the additive drift theorem.
\begin{proposition}
    \label{proposition:var:bound}
    Suppose $\din \geq 2$. Then, there exists $c > 0$ such that, for arbitrary $\m[t]$ and $\sig[t]$, it holds
    \begin{align}
        \mathrm{Var} \left[ \log\left( \frac{\| \m[t+s] \|}{\sig[t+s]} \right) \mid \mathcal{F}_t \right] \leq c \enspace.
    \end{align}
\end{proposition}



Then, by applying Theorem~\ref{theorem:additive-drift}, we show that the search distribution prematurely converges with probability 1 on a high-dimensional problem. In the derivation, we consider the stochastic process $\bar{\theta}_\tau = (\m[s  \tau], \sig[s  \tau]) $ that consists of the stochastic process $\theta_t = (\m[t], \sig[t])$ for every $s$ iteration with the filtration $\bar{\mathcal{F}}_{\tau} = \mathcal{F}_{s \tau}$.
\begin{proposition}
    \label{proposition:ratio:as}
    Suppose $\pinmut \in \Theta(1/\din)$, $\mm[0] = \mathbf{0}$, and $\sigma_0 = (\sigLB)^{\frac{K}{s-1}}$ for some $K \in \mathbb{Z}$. \new{Suppose $\sigLB$ is given by~\eqref{eq:sigLB}.} Then, when $\dco$ satisfies~\eqref{eq:cond:dco}, there exist $D \in \mathbb{N}$ and $\delta > 0$ such that, for all $\din \geq D$, the following holds with probability 1:
    \begin{align}
        \log\left( \frac{\| \m[s \tau] \|}{\sig[s \tau]} \right) - \log\left( \frac{\| \m[0] \|}{\sig[0]} \right) \geq \tau \delta - o\left( \tau^{0.5 + \epsilon} \right)
    \end{align}
    for \rev{all $\tau \in \mathbb{N}$ and} all $\epsilon > 0$.
    In other words, with probability 1,
    \begin{align}
        \frac{\| \m[s \tau] \|}{\sig[s \tau]} \in \Omega \left( \exp \left( \tau\right) \right)
        \enspace.
    \end{align}
\end{proposition}

Finally, we derive that there exists a general lower bound $C_\mathrm{lower} > 0$ of the expected distance $\E[ \| \m[s \tau] \| ]$ of the mean vector and the optimal solution for any $\tau \in \mathbb{N}$.

\begin{theorem}
    \label{theorem:premature:logm}
    Consider (1+1)-LB-ES with $1/s$-success rule on \obj{}.
    Suppose the same assumption as in Proposition~\ref{proposition:ratio:as}.
    Then, when $\dco$ satisfies~\eqref{eq:cond:dco}, 
    there exists $D \in \mathbb{N}$ such that, for all $\din \geq D$, 
    \begin{align}
        \label{eq:mlog:lower}
        \E\left[ \log \left( \frac{ \| \m[s \tau +1] \| }{\| \m[s \tau] \|} \right) \mid \mathcal{F}_{s \tau} \right] \in - O\left( \exp \left( - 3 \tau \right) \right) \enspace.
    \end{align}
    In addition, there exists a finite constant $C_\mathrm{lower} > 0$ such that, for any $\tau \in \mathbb{N}$, 
    \begin{align}
    \E\left[ \| \m[s \tau] \| \right] \geq C_\mathrm{lower} \cdot \| \m[0] \| \enspace.
\end{align}
\end{theorem}

Theorem~\ref{theorem:premature:logm} reveals that the (1+1)-LB-ES failed to continue decreasing the distance between the mean vector and the optimal solution in expectation.
This implies that setting a lower bound of the standard deviation for integer variables is not sufficient to obtain linear convergence on high-dimensional mixed-integer optimization problems.

One may consider that reducing the lower bound $\sigLB$ can resolve the premature convergence of (1+1)-LB-ES. 
Although the smaller $\sigLB$ may resolve the premature convergence, we consider the optimization performance of the integer variable may be worsened.
We note that our analysis focuses on the optimization process that starts with the optimal integer variable, and the lower bound $\sigLB$ setting has been tuned to ensure the performance of the integer variable.





\subsection{Linear Convergence of (1+1)-LUB-ES}
\label{sec:lub-es}

For achieving linear convergence on high-dimensional mixed-integer optimization problems, this section investigates (1+1)-LUB-ES to evaluate the effect of the additional upper bound on integer handling and shows its linear convergence.

Our analysis is based on the analysis tools used for the runtime analysis of (1+1)-ES on the sphere function~\cite{akimoto:gecco:2018}.
We note that although the reference~\cite{akimoto:gecco:2018} focuses on the $1/5$-success rule, we can extend the analysis result to $1/s$-success rule.
According to~\cite{akimoto:gecco:2018}, we introduce the normalized step-size $\bar{\sigma}_t = \dco \cdot \sigma / \| \m \|$ and the normalized success probability with rate $r$ for continuous part as
\begin{align}
    p^{\mathrm{succ,co}}_{r,\dco}(\bar{\sigma}) &:= \Pr_{\x \sim \mathcal{N}(\m[], \sig^2 \mathbf{I}_{\dco}) } \left( \| \x \| \leq (1 - r) \| \m[] \| \right) \\
    &= \Pr\left( \left\| \boldsymbol{e}_1 + \frac{ \bar{\sigma} }{ \dco } \mathcal{N} \right\| \leq 1 - r \right) \enspace,
\end{align}
where $\boldsymbol{e}_1 = (1, 0, \cdots, 0)$ and $\mathcal{N}$ is a $\dco$-dimensional vector obeying the standard normal distribution.
\citet[Lemma~3.1,~3.2]{akimoto:gecco:2018} show several properties of $p^{\mathrm{succ,co}}_{r,\dco}(\bar{\sigma})$. For example, it is strictly monotonically decreasing.
We also develop the potential function $V(\theta)$ for $\theta = \{ \m[], \sig[], \langle \D[] \rangle \}$ as
\begin{align}
    \label{eq:potential}
    V(\theta) = V_{\mathrm{co}}(\theta) + v_{\mathrm{in}} \cdot \log \left( \frac{ \sig[] \langle \D[] \rangle }{ \sigLB } \right) \enspace,
\end{align}
where $V_{\mathrm{co}}(\theta)$ is the potential function defined in~\cite{akimoto:gecco:2018} as
\begin{multline}
    V_{\mathrm{co}}(\theta) = \log (\| \m[] \|) \\
    + v \cdot \log^{+}_{\max} \left\{  
    \frac{\alpha \cdot \ell \cdot \| \boldsymbol{m} \|}{ \dco \cdot \sigma} , 
    \frac{\alpha^{\frac{1}{s-1}} \cdot \dco \cdot \sigma}{ u \cdot \| \boldsymbol{m} \| } \right\} \enspace,
\end{multline}
where $\log^{+}_{\max}\{a,b\} = \max\{0, \log (a), \log(b)\}$.
The parameters $\ell$ and $u$ are determined using two probabilities $p_u$ and $p_\ell$ satisfying $u / \ell > \alpha^{\frac{s}{s-1}}$ and $0 < p_u < 1/s < p_{\ell} <  1 / 2$ such that $p^{\mathrm{succ,co}}_{0,\dco}(u) = p_u$ and $p^{\mathrm{succ,co}}_{0,\dco}(\ell) = p_\ell$. 
Because $V(\theta_t) \geq \log (\| \m \|)$, we will show the hitting time of $\log (\| \m \|)$ by deriving the hitting time of $V(\theta_t)$ using the truncated drift theorem~\cite[Theorem~2.1]{akimoto:gecco:2018} (stated as Theorem~\ref{theorem:truncated-drift} in Appendix~\ref{apdx:existing}).

We derive the upper bound of the truncated drift $\mathbb{E}[ \max\{ V(\theta_{t+1}) - V(\theta_{t}), -A \} \mid \mathcal{F}_t ]$ for $A > 0$ by considering the first and second terms of the potential function in~\eqref{eq:potential}.
The first term can be treated using the truncated drift for $V_{\mathrm{co}}(\theta)$, which is derived in~\cite{akimoto:gecco:2018}.
We then investigate the second term to obtain the upper bound of the truncated drift for $V(\theta)$.
Considering the modified update rule of $\langle \D[t] \rangle$ in lines \ref{algo:line:upper-1} and \ref{algo:line:upper-2} of Algorithm~\ref{alg:mixint-ES-LUB}, the following proposition is obviously obtained.
\begin{proposition}
    \label{proposition:integer-mutation:decrease}
    Consider (1+1)-LUB-ES with $1/s$-success rule.
    Then, $\sig[t] \langle \D[t] \rangle$ is monotonically decreasing as
    \begin{align}
        - \frac{\log (\alpha)}{s-1} \leq \log \left( \frac{ \sig[t+1] \langle \D[t+1] \rangle }{ \sig[t] \langle \D[t] \rangle } \right) \leq 0 \enspace.
    \end{align}
\end{proposition}

For analysis, we introduce the normalized success probabilities for the integer part and for the next candidate solution as
\begin{align}
    p^{\mathrm{succ}}_{\mathrm{in}}( \sigma_{\mathrm{in}} ) &:= \Pr \left( \Int\left[  \sigma_{\mathrm{in}}\mathcal{N}_{\mathrm{in}}\right] = \boldsymbol{0} \right) \\
    &= \left( 1 - 2 \Phi \left( - \frac{1}{2 \sigma_{\mathrm{in}}} \right) \right)^{\din} \enspace,
\end{align}
where $\mathcal{N}_{\mathrm{in}}$ is a $\din$-dimensional vector obeying the standard normal distribution.
In the following, $p^{\mathrm{succ}}_{\mathrm{in}}(\sigLB) = ( 1 - \new{\pinmut} )^{\din}$ is denoted as $p^{\mathrm{succ}}_{\mathrm{in,LB}}$.
Then, Propositions~\ref{proposition:cdf-ratio} and~\ref{proposition:cdf-bound} show
\begin{align}
    \label{eq:psucc:in:cases}
    p^{\mathrm{succ}}_{\mathrm{in}}(\sigma_{\mathrm{in}}) = \begin{cases}
        p^{\mathrm{succ}}_{\mathrm{in,LB}} & \text{if } \sigma_{\mathrm{in}} = \sigLB \\
        O \left( \din^{- \gamma} \right) & \text{if } \sigma_{\mathrm{in}} \geq \sigLB \cdot \alpha^{\frac{1}{s-1}} 
    \end{cases}
    \enspace.
\end{align}
This implies that the elitist solution is rarely updated until it holds $ \sig[t] \langle \D[t] \rangle = \sigLB$.
We note that generating a non-optimal integer leads to an unsuccessful update, where $V_\mathrm{co}$ is updated as
\begin{align}
    V_\mathrm{co}(\theta_{t+1}) &= V_\mathrm{co}(\{ \m[t], \sig \cdot \alpha^{- \frac{1}{s-1}} \}) =: V^{\downarrow}_\mathrm{co}(\theta_{t}) \enspace.
\end{align}
In addition, because the continuous and integer parts are independently generated, and because the generation process of the continuous part is the same as (1+1)-ES, when $\sig[t] \langle \D[t] \rangle = \sigLB$, we have
\begin{align}
\begin{split}
    &\mathbb{E}[ \max\{ V(\theta_{t+1}) - V(\theta_{t}), -A \} \mid \mathcal{F}_t ] \\
    &= \mathbb{E}_{\mathrm{ES}} \left[ \max\left\{ V_\mathrm{co}(\theta_{t+1}) - V_\mathrm{co}(\theta_{t}), -A \right\} \mid \mathcal{F}_t \right] \cdot p^{\mathrm{succ}}_{\mathrm{in,LB}} \\
    &\qquad + \max\left\{ V^{\downarrow}_\mathrm{co}(\theta_{t}) - V_\mathrm{co}(\theta_{t}), -A \right\} \cdot \left( 1 - p^{\mathrm{succ}}_{\mathrm{in,LB}} \right)
    \enspace,
\end{split}
\end{align}
where $\mathbb{E}_{\mathrm{ES}}$ is the expectation under the update of (1+1)-ES.
When $\sig[t] \langle \D[t] \rangle \geq \sigLB \cdot \alpha^{- \frac{1}{s-1}}$, the second term in~\eqref{eq:potential} decreases in expectation.
Then, we have the upper bound of the truncated drift as follows.
\begin{proposition}
    \label{proposition:integer-mutation:drift}
    Consider (1+1)-LUB-ES with $1/s$-success rule on \obj{}.
    Suppose $\pinmut \in \Theta(1/\din)$, $\mm[0] = \mathbf{0}$, and $\sigma_0 = (\sigLB)^{\frac{K}{s-1}}$ for some $K \in \mathbb{Z}$. 
    \new{Suppose $\sigLB$ is given by~\eqref{eq:sigLB} and $\pinmut \in O(1/\dco)$.}
    Then, if $v$ and $A$ fulfill $0 < v < \min\{1, A_{\mathrm{co}} / \log (\alpha) \}$, it holds
    \begin{align}
        \mathbb{E}\left[ \max\{ V(\theta_{t+1}) - V(\theta_{t}), -A \} \mid \mathcal{F}_t \right] \leq - B \enspace,
    \end{align}
    where 
    \begin{align}
        A &= A_{\mathrm{co}} + \frac{v_\mathrm{in}}{s-1} \cdot \log (\alpha) \\
        B &= \min \left\{ B_1, B_2, B_3, B_4 \right\}
    \end{align}
    defined with
    \begin{align*}
        & B_1 = \frac{v_{\mathrm{in}} - v}{s-1} \cdot \log (\alpha) \cdot \left( 1 - O \left( \din^{- \gamma} \right) \right) - \min\{B_2', B_3', B_4'\} \cdot O \left( \din^{- \gamma} \right) \\
        & B_2' = \left( A_\mathrm{co} \cdot p^\ast - \frac{s}{s-1} \cdot v \cdot \log(\alpha) \right) 
        \enspace, \quad 
        B_2 = B_2' \cdot p^{\mathrm{succ}}_{\mathrm{in,LB}} \\
        & B_3' = \frac{v \cdot \log(\alpha)}{s-1} \cdot \left( s \cdot p_{\ell}  - 1 \right) \\
        & B_3 = \frac{v \cdot \log(\alpha)}{s-1} \cdot \left( s \cdot p_{\ell} \cdot p^{\mathrm{succ}}_{\mathrm{in,LB}} - 1 \right) \\
        & B_4' = \frac{v \cdot \log(\alpha)}{s-1} \cdot ( 1 - s \cdot p_u ) 
        \enspace, \quad 
        B_4 = B_4' \cdot p^{\mathrm{succ}}_{\mathrm{in,LB}}
    \end{align*}
    and $p^\ast = \min_{\bar{\sigma} \in [\ell,u]} \left\{ p^{\mathrm{succ,co}}_{r^{\ast},\dco}(\bar{\sigma}) \right\}$ with $r^{\ast} = 1 - \exp\left( - \frac{A_{\mathrm{co}}}{1-v} \right)$.
\end{proposition}

Finally, tuning the parameters in the potential, we derive the linear convergence of (1+1)-LUB-ES in the next theorem.

\begin{theorem}
    \label{theorem:integer-mutation:ert}
    Suppose $\dco \geq 2$ and there exist $D_{p} \in \mathbb{N}$ and $0 < \bar{p}_\mathrm{in} < 1$ such that $p^{\mathrm{succ}}_{\mathrm{in,LB}} \geq \bar{p}_\mathrm{in} \in \Theta(1)$ for all $\din \geq D_{p}$. Suppose $s$ is given to fulfill $s > 2 /  p^{\mathrm{succ}}_{\mathrm{in,LB}}$.
    Let $p' = \min_{\bar{\sigma} \in [\ell,u]} \left\{ p^{\mathrm{succ,co}}_{r',\dco}(\bar{\sigma}) \right\}$ and $r' = 1 - \exp\left( - \frac{ \log (\alpha) }{\dco \cdot \log (\alpha) - 1} \right)$. 
    Then, for 
    \begin{align}
        A_\mathrm{co} = \frac{1}{\dco}, \quad v = \frac{p'}{2 \dco \cdot \log(\alpha)}, \quad \text{and} \quad v_{\mathrm{in}} = 2 v \enspace,
    \end{align}
    there exist $D \in \mathbb{N}$ satisfying the following for all $\din \geq \max\{ D_{p}, D\}$:
    \begin{itemize}[left=5pt]
        \item $B > 0$ and $B \in \Theta( 1 / \dco)$.
        \item For any $T \in \mathbb{N}$, $\mathbb{E}[ \log ( \| \m[T] \| )] \leq  V(\theta_0) - B \cdot T$.
        \item For any $\epsilon \in \mathbb{R}$, the expectation of the hitting time $T_\epsilon := \min \{ t: \|\m\| \leq \epsilon \}$ holds
        \begin{align}
            \mathbb{E}\left[ T_\epsilon \right] \leq \frac{V(\theta_0) - \log(\epsilon) + \frac{1}{\dco} + \frac{p'}{\dco \cdot (s-1)}}{B} \enspace,
        \end{align}
        in particular, $\mathbb{E}\left[ T_\epsilon \right] \in \Theta(\dco \cdot \log(1/\epsilon))$.
    \end{itemize}
\end{theorem}

\section{Numerical Experiments}
\begin{figure*}[t]
    \centering
    \includegraphics[width=0.99\linewidth]{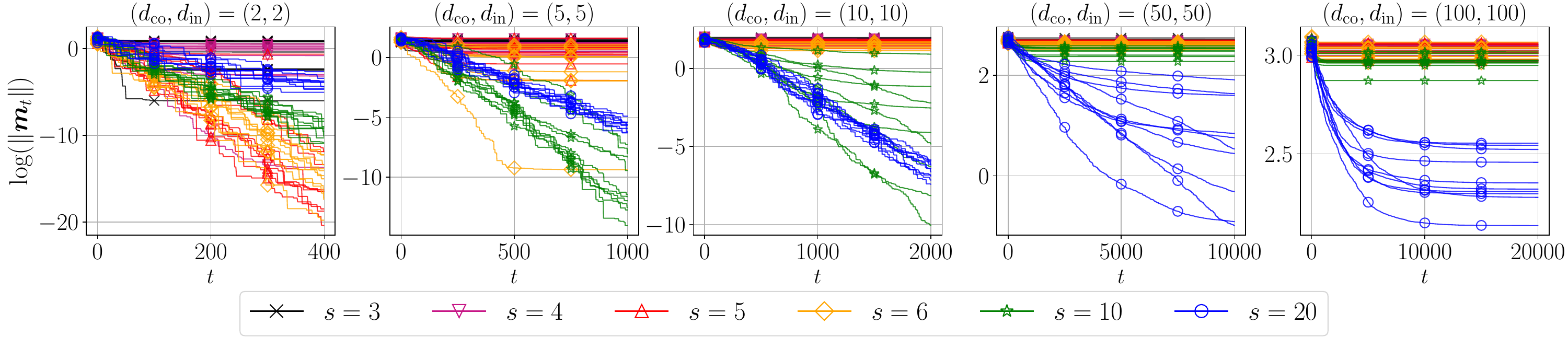}
    \vspace{-2mm}
    \caption{Logarithmic distance to the optimum $\log (\| \boldsymbol{m}_t \|)$ for (1+1)-LB-ES on \textsc{LexicoSphereInt}.}
    \label{fig:LB}
\end{figure*}

\begin{figure*}[t]
    \centering
    \includegraphics[width=0.99\linewidth]{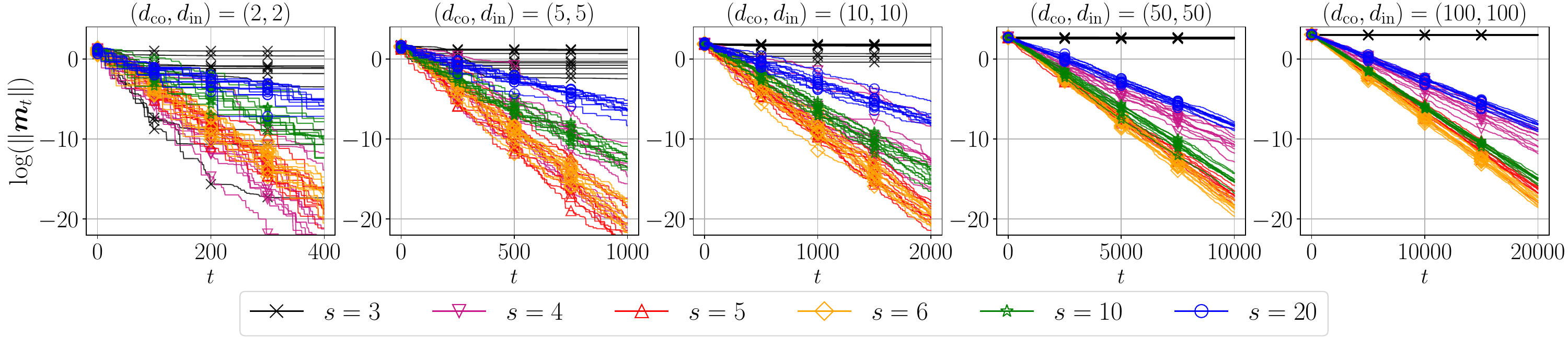}
    \vspace{-2mm}
    \caption{Logarithmic distance to the optimum $\log (\| \boldsymbol{m}_t \|)$ for (1+1)-LUB-ES on \textsc{LexicoSphereInt}.}
    \label{fig:LUB}
\end{figure*}

\begin{figure*}[t]
    \centering
    \includegraphics[width=0.95\linewidth]{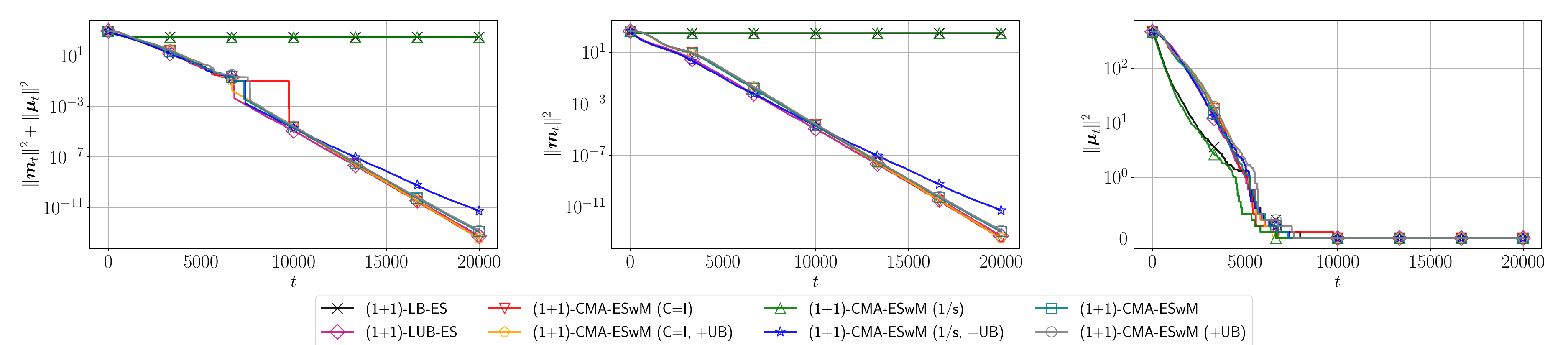}
    \vspace{-1mm}
    \caption{The average $\| \m \|^2 + \| \mm \|^2$, $\| \m \|^2$, and $\| \mm \|^2$ on \textsc{SphereInt} with $(\dco, \din) = (100, 100)$ over 10 independent trials.}
    \label{fig:SphereInt}
\end{figure*}

In this section, we investigate the influence of the finite number of dimensions and the parameter $s$ in (1+1)-LB-ES and (1+1)-LUB-ES through empirical evaluations.
\rev{The goal of the experiments is to illustrate the finite-dimensional behavior predicted by the theory and to visualize the premature-convergence mechanism.}
Figure~\ref{fig:LB} shows the logarithm of the distance $\|\m[t]\|$ between the mean vector of (1+1)-LB-ES and the optimal solution on \textsc{LexicoSphereInt} for continuous variables, varying $(\dco, \din)$ and $s$ with $\alpha=1.5$ and $\pinmut = 1/(\dco+\din)$.
For each combination of $(\dco, \din)$ and $s$, ten independent runs were conducted, and the trajectories from all runs are presented in Figure~\ref{fig:LB}.
In each run, the initial mean vector $\m[0]$ was uniformly sampled from $[1, 3]^{\dco}$, and the other initial values were set to $\mm[0]=\boldsymbol{0}$, $\sig[0]=1$ and $\langle \D[0] \rangle_1 = \cdots = \langle \D[0] \rangle_{\din} = 1$.
We can see that increasing the dimension leads to premature convergence, and this observation is consistent with Theorem~\ref{theorem:premature:logm}.

Figure~\ref{fig:LUB} shows the results for (1+1)-LUB-ES.
The experimental settings for the (1+1)-LUB-ES were the same as those described previously for the (1+1)-LB-ES.
These results indicate that (1+1)-LUB-ES can avoid premature convergence by setting $s$ greater than a certain value, which is consistent with the condition $s > 2 /  p^{\mathrm{succ}}_{\mathrm{in,LB}}$ established in Theorem~\ref{theorem:integer-mutation:ert}. 
Considering that, in the present case where $\dco = \din$, we have $2 /  p^{\mathrm{succ}}_{\mathrm{in,LB}} = 2/(1-\frac{1}{\dco+\din})^{\din} \leq 4$, the derived condition on $s$ may be considered practical for hyperparameter tuning.
Moreover, by observing the slope of the plot, it can be seen that (1+1)-LUB-ES exhibits linear convergence, and the expected first hitting time increases proportionally with $\dco$.

To complement our theoretical analysis, we empirically investigate the behavior on the standard \textsc{SphereInt} function, examining the optimization process starting from non-optimal integer values.
We also investigate (1+1)-CMA-ES with Margin ((1+1)-CMA-ESwM) and its variants to clarify the conditions leading to premature convergence.
Specifically, we test variants substituting the step-size adaptation with the $1/s$-success rule (\textsf{1/s}) and disabling covariance matrix adaptation (\textsf{C=I}), applying the upper bound strategy on the margin correction (\textsf{+UB}) to each.
For $1/s$-success rule, we set $s=5$.
In each run, $\m[0]$ and $\mm[0]$ were uniformly sampled from $[1, 3]^{\dco}$ and $\{1, 2, 3\}^{\din}$, respectively.
Figure~\ref{fig:SphereInt} shows the trajectories of the average best evaluation values.
We can see that (1+1)-LB-ES suffers from premature convergence in the continuous variables despite finding the optimal integers.
In contrast, (1+1)-LUB-ES achieves linear convergence.
Notably, for (1+1)-CMA-ESwM, stagnation is avoided either by employing the original step-size adaptation with smoothed success probabilities or by incorporating the upper bound.
This indicates that preventing excessive step-size increases triggered by inferior integer mutations is the essential factor in resolving the premature convergence issue.


\section{Conclusion}

We have theoretically investigated the runtime of (1+1)-LB-ES and (1+1)-LUB-ES with $1/s$-success update on \obj{} starting with the optimal integer value.
Theorem~\ref{theorem:premature:logm} reveals that (1+1)-LB-ES suffers from premature convergence, i.e., there exists a finite constant $C_\mathrm{lower} > 0$ such that $\E\left[ \| \m[s \tau] \| \right] \geq C_\mathrm{lower} \cdot \| \m[0] \|$ for any $\tau \in \mathbb{N}$. On the other hand, Theorem~\ref{theorem:integer-mutation:ert} shows that the runtime of (1+1)-LUB-ES for finding an $\epsilon$-neighbor is $\mathbb{E}\left[ T_\epsilon \right] \in \Theta(\dco \cdot \log(1/\epsilon))$. 
These results imply that the additional upper bound, which is incorporated into (1+1)-LUB-ES, is essential for high-dimensional mixed-integer optimization.
Theorem~\ref{theorem:integer-mutation:ert} also reveals the condition on the target success rate $1/s$ required for the linear convergence. The numerical simulation shows premature convergence of the (1+1)-LB-ES even with a moderate number of dimensions, whereas (1+1)-LUB-ES converges to the optimal solution with suitable settings for $s$.

As future work, we will conduct the runtime analysis for integer variables, which has the potential to ensure the optimization performance of mixed-integer optimization by combining the results derived in this paper.
In addition, investigating the effect of the hyperparameters, such as $s$ and $\sigLB$, on the convergence rate is a promising direction for future work, which may lead to the performance improvement of (1+1)-LUB-ES. We also believe that our analysis approach is applicable to mixed-category optimization methods~\cite{CatCMA:2024,CatCMAwM:2025}, potentially broadening the impact of this research.

\begin{acks}
This study was partially funded by JSPS KAKENHI (JP23K28156, JP23H00491, and JP24K20857) and JST ACT-X (JPMJAX24C7).
Generative AI tools were used to improve the English writing of this manuscript.
\end{acks}

\bibliographystyle{ACM-Reference-Format}
\bibliography{sample-base}

\appendix

\setcounter{secnumdepth}{1} 
\section{Existing Theorems and Propositions} \label{apdx:existing}
\begin{theorem}[Proposition~3.1 in~\cite{AdditiveDtift:ECJ:2003}, Theorem~5.8 in~\cite{AllDrift:arxiv:2024}]
    \label{theorem:additive-drift}
    Let $(X_t)_{t\in\mathbb{N}}$ be an integrable random process over $\mathbb{R}$ with $X_0 = 0$. Furthermore, suppose that there are $\delta > 0$ and $c > 0$ such that, for all $t \in \mathbb{N}$, it holds the following:
    \begin{align}
        \mathbb{E}[ X_t - X_{t+1} \mid \mathcal{F}_t ] &\geq \delta \\
        \mathrm{Var}[ X_t - X_{t+1} \mid \mathcal{F}_t ] &\leq c \enspace.
    \end{align}
    Then, for all $\epsilon > 0$, the following holds with probability 1.
    \begin{align}
        X_t \geq t \delta - o\left( t^{0.5 + \epsilon} \right) 
    \end{align}
\end{theorem}

\begin{theorem}[Theorem~2.1 in~\cite{akimoto:gecco:2018}]
    \label{theorem:truncated-drift}
    Let $(X_t)_{t\in\mathbb{N}}$ be an integrable random process over $\mathbb{R}$ with $X_0 = x_0$. 
    For $A > 0$, let $(Y^A_t)_{t\in\mathbb{N}}$ be a truncated process defined iteratively as
    \begin{align}
        Y^A_0 = x_0 
        \quad \text{and} \quad
        Y^A_{t+1} = Y^A_t + \max \left\{ X_{t+1} - X_t, - A \right\}
        \enspace.
    \end{align}
    For $\beta < x_0$, let $T^X_\beta = \min \{ t: X_t \leq \beta \}$ be the first hitting time of the set $(-\infty, \beta]$.
    If there exist $A,B > 0$ such that $Y^A_t$ is integrable, i.e., $\mathbb{E}[ | Y^A_t | ] < \infty$, and
    \begin{align}
        \mathbb{E} \left[ Y^A_{t+1} - Y^A_{t} \mid \mathcal{F}_t \right] 
        &= \mathbb{E} \left[ \max \left\{ X_{t+1} - X_t, -A \right\} \mid \mathcal{F}_t \right] \\
        &\leq - B \enspace,
    \end{align}
    then the expectation of $T^X_{\beta}$ satisfies
    \begin{align}
        \mathbb{E} \left[ T^X_{\beta} \right] 
        \leq \mathbb{E} \left[ T^{Y^A}_{\beta} \right]
        \leq \frac{x_0 - \beta + A}{B} 
        \enspace.
    \end{align}
\end{theorem}

\section{Proofs of Propositions} \label{apdx:proofs-propositions}

This section provides the proofs of propositions in our analysis.
The following subsections show the proofs for Sections~\ref{sec:properties:gaussian}, \ref{sec:lb-es}, and \ref{sec:lub-es}, respectively.

\subsection{Proof of Propositions about Normal Distribution} \label{apdx:sec:proof:normal}

\begin{proof}[Proof of Proposition~\ref{proposition:gaussian:cdf-ratio}]
    Denoting 
    \begin{align}
        h(\theta) = \Phi(-\theta) \cdot \exp( \theta^2 / 2 ) \enspace,
    \end{align}
    the claim can be written as $h(\theta) < h(\theta / \beta) $. Therefore, the claim holds if $\log h$ is monotonically decreasing on $(0, \infty)$. Because we have
    \begin{align}
        \frac{\mathrm{d} \log h(\theta)}{\mathrm{d} \theta} = - \frac{\Phi'(-\theta)}{\Phi(-\theta)} + \theta = - \frac{\exp(- \theta^2 / 2)}{\int_{-\infty}^{\rev{\theta}} \exp(- t^2 / 2) \mathrm{d}t} + \theta \enspace,
    \end{align}
    the inequality $\frac{\mathrm{d} \log h(\theta)}{\mathrm{d} \theta} < 0$ holds when
    \begin{align}
        g(\theta) := \int_{-\infty}^{\rev{\theta}} \exp \left( - \frac{t^2}{2} \right) \mathrm{d}t - \frac{1}{\theta} \exp \left(- \frac{\theta^2}{2} \right) < 0 \enspace.
    \end{align}
    We can see that $\lim_{\theta \to - \infty} g(\theta) = 0$ and $g(\theta)$ is monotonically increasing as
    \begin{align}
        \frac{\mathrm{d} g(\theta)}{\mathrm{d} \theta} = \frac{1}{\theta^2} \exp \left( - \frac{\theta^2}{2} \right) > 0 \enspace.
    \end{align}
    This finishes the proof.
\end{proof}


\begin{proof}[Proof of Proposition~\ref{proposition:cdf-bound}]
    First, we show the upper bound of $\Phi^{-1}(y)$. For an invertible lower bound $g(x)$ for $\Phi(x)$, we have $\Phi^{-1}(y) < g^{-1}(y)$. A lower bound of $1 - \Phi(x)$ for $x > 0$ is given by~\cite{MillsRatio:1941} as
    \begin{align} 
        1 - \Phi(x) \geq \frac{x}{x^2+1} \cdot \phi(x) \enspace,
    \end{align}
    where $\phi(x) = \frac{1}{\sqrt{2 \pi}} \exp \left( - \frac{x^2}{2} \right)$.
    Considering $1 - \Phi(x) = \Phi(-x)$ and $\phi(x) = \phi(-x)$, for $x < -1$,
    \begin{align}
        \Phi(x) \geq \frac{-x}{(-x)^2+1} \cdot \phi(x) \geq \sqrt{\frac{e}{8\pi}} \exp ( - x^2 )
    \end{align}
    and we set as $g(x) = \sqrt{\frac{e}{8\pi}} \exp ( - x^2 )$. For $g$ on $(- \infty, -1)$, its inverse function is given by
    \begin{align}
        g^{-1}(y) = - \sqrt{ - \log \left( \sqrt{\frac{8\pi}{e}} \cdot y \right) } \enspace.
    \end{align}
    We note that, because $\exp( z^2 )$ is monotonically decreasing for $z < -1$, we have
    \begin{align}
        \exp\left( (\Phi^{-1}(y))^2 \right) 
        &\geq \exp\left( (g^{-1}(y))^2 \right) \\
        &= \exp\left( - \log \left( \sqrt{\frac{8\pi}{e}} y \right) \right) \\
        &= \left( \sqrt{\frac{8\pi}{e}} \cdot y \right)^{- 1}
    \end{align}
    This is the end of the proof.
\end{proof}

\begin{proof}[Proof of Proposition~\ref{proposition:chi-squared:inverse:trunc}]
    We note that $\log (1 + x) \geq \frac{x}{1 + x}$ for $-1 < x$ (see~\cite{log-bound:2007} for example).
    Because \del{$0 < \| \m[t+1] \|^2 / \| \m[t] \|^2 \leq 1$}{}\new{$0 < \min\{ \| \x \|, \|\m[]\| \} / \| \m[] \| \leq 1$}, we have
    \begin{align*}
        \log \left( \frac{\min\{ \| \x \|, \|\m[]\| \}}{\| \m[] \|} \right)
        &= \log \left( \frac{\| \x \|}{\| \m[] \|} \right) \cdot \mathbf{1}_{\{ \| \x \| \leq \| \m[] \| \}} \\
        &= \frac{1}{2} \cdot \log \left( \frac{\| \x \|^2}{\| \m[] \|^2} \right) \cdot \mathbf{1}_{\{ \| \x \| \leq \| \m[] \| \}} \notag \\
        & \geq \frac{1}{2} \cdot \frac{\| \m[] \|^2}{\| \x \|^2} \cdot \left( \frac{\| \x \|^2}{\| \m[] \|^2} - 1 \right) \cdot \mathbf{1}_{\{ \| \x \| \leq \| \m[] \| \}} 
        \\
        & = \frac{1}{2} \cdot \left( 1 - \frac{\| \m[] \|^2}{\| \x \|^2} \right) \cdot \mathbf{1}_{\{ \| \x \| \leq \| \m[] \| \}} 
        \enspace.
    \end{align*}

    \def\thefootnote{\fnsymbol{footnote}}
    In the following, we denote $\x = \m[] + \sig[] \cdot \bxi$ with $\bxi \sim \mathcal{N}(\mathbf{0}, \mathbf{I}_d)$ and $Z = \| \x \|^2 / \sig[]^2 = \left\| \frac{ \m[] }{ \sig[] } + \bxi \right\|^2$. Then, the distribution of $Z$ follows the non-central $\chi^2$-distribution with degree of freedom $k = d$ and non-centrality parameter $\lambda = \| \m[] \|^2 / \sig[]^2$.
    The probability density function of the non-central $\chi^2$-distribution is 
    \begin{align}
        f_{\mathrm{n}}(Z; k, \lambda) =  \exp \left( - \frac{\lambda}{2} \right) \cdot \sum^\infty_{j=0} \frac{(\lambda / 2)^j}{j !} \cdot f_{\mathrm{c}}(Z; k + 2j) \enspace,
    \end{align}
    where $f_{\mathrm{c}}(z; k) = \frac{2^{- k / 2}}{\Gamma(k/2)} z^{k/2 -1} e^{- z/2}$ is the probability density function of the $\chi^2$-distribution with degree of freedom $k$, defined with the gamma function $\Gamma(x) = \int^\infty_{t=0} t^{x-1} e^{\rev{-t}} \mathrm{d} t$. 
    Then, denoting $A_j = \frac{d}{2} + j$ and $B = \frac{\lambda}{2}$, we have
    \begin{align}
        &\mathbb{E} \left[ \left( 1 - \frac{\lambda}{Z} \right) \cdot \mathbf{1}_{\{ Z \leq \lambda \}} \right] \notag \\
        &\quad = \int_{0}^\lambda \left( 1 - \frac{\lambda}{Z} \right) \cdot f_{\mathrm{n}}(Z; k, \lambda) \mathrm{d}Z \\ 
        &\quad \overset{\footnotemark[2]{}}{=} e^{ - B} \cdot \sum^\infty_{j=0} \frac{B^j}{j !} \cdot \frac{ \gamma(A_j, B) - B \cdot \gamma(A_j - 1, B) }{ \Gamma ( A_j ) } \enspace,
    \end{align}
    where $\gamma$ is the incomplete gamma function.\footnotetext[2]{When $\{ f_i \}$ is a set of non-negative measurable functions, the integral and the infinite sum can be interchanged, i.e., $\sum^\infty_{i=0} \int f_i(x) \mathrm{d}x = \int \sum^\infty_{i=0} f_i(x) \mathrm{d}x$.}
    Because $\gamma(a+1,x) = a \cdot \gamma(a,x) - x^a e^{-x}$, it holds
    \begin{align}
        \gamma(A_j - 1, B) 
        &= \frac{\gamma(A_j, B) + B^{A_j-1} e^{-B}}{A_j - 1} \\
        &= \frac{\gamma(A_j + 1, B) + B^{A_j} e^{-B}}{(A_j - 1) A_j} + \frac{B^{A_j-1} e^{-B}}{A_j - 1}
        \enspace.
    \end{align}
    Because $A_j - 1 \geq j+1$ for $d \geq 4$ and $A_j \cdot \Gamma ( A_j ) = \Gamma ( A_j +1 )$, we have
    \begin{align}
        \begin{split}
            & \mathbb{E} \left[ \left( 1 - \frac{\lambda}{Z} \right) \cdot \mathbf{1}_{\{ Z \leq \lambda \}} \right] \geq e^{-B} \cdot \sum^\infty_{j=0} \frac{B^j}{j !} \cdot \frac{ \gamma(A_j, B) }{ \Gamma ( A_j ) } \\
            &\qquad - e^{-B} \cdot \sum^\infty_{j=0} \frac{B^j}{j !} \cdot \frac{B}{j+1} \cdot \frac{ \gamma(A_j + 1, B) }{ \Gamma ( A_j + 1 ) } \\
            &\qquad  - e^{-2B} \cdot \sum^\infty_{j=0} \frac{B^j}{j !} \cdot \frac{B^{A_j-1}}{(A_j-1) \cdot \Gamma ( A_j )} \cdot \left( \frac{B}{A_j} + 1 \right) \enspace.
        \end{split}
    \end{align}
    Denoting $J_j = \frac{B^j}{j !} \cdot \frac{ \gamma(A_j, B) }{ \Gamma ( A_j ) }$, the sum of the first and second term is bounded as
    \begin{align}
        e^{-B} \cdot \left( \sum^\infty_{j=0} J_j - \sum^\infty_{j=0} J_{j+1} \right) = e^{-B} \cdot \frac{ \gamma(A_0, B) }{ \Gamma ( A_0 ) } > 0 \enspace.
    \end{align}
    Because
    \begin{align}
        \frac{B^{A_j-1}}{(A_j-1) \cdot \Gamma ( A_j )} \cdot \frac{B}{A_j} 
        &\leq \frac{B^{A_j}}{(j+1) \cdot \Gamma ( A_j + 1)} \quad \text{and}\\
        \frac{B^{A_j-1}}{(A_j-1) \cdot \Gamma ( A_j )} 
        &\leq \frac{B^{A_j-1}}{(j+1) \cdot \Gamma ( A_j)} \enspace,
    \end{align}
    we have
    \begin{align}
        &\sum^\infty_{j=0} \frac{B^j}{j !} \cdot \frac{B^{A_j-1}}{(A_j-1) \cdot \Gamma ( A_j )} \cdot \left( \frac{B}{A_j} + 1 \right) \notag \\
        &\leq \frac{1}{B} \cdot \sum^\infty_{j=0} \frac{B^{j+1}}{(j+1)!} \cdot \frac{B^{A_j}}{\Gamma ( A_j + 1)} + \frac{1}{B} \cdot \sum^\infty_{j=0} \frac{B^{j+1}}{(j+1)!} \cdot \frac{B^{A_j - 1}}{\Gamma ( A_j )} \\
        &= \frac{1}{B}\left( I_{\frac{d}{2}-1}(2B) - \frac{B^{\frac{d}{2} - 1}}{\Gamma (d /2)} \right) + \frac{1}{B}\left( I_{\frac{d}{2}-2}(2B) - \frac{B^{\frac{d}{2} - 2}}{\Gamma (d /2 -1)} \right) \\
        &\leq \frac{1}{B}\left( I_{\frac{d}{2}-1}(2B) + I_{\frac{d}{2}-2}(2B) \right) \enspace.
    \end{align}
    where $I_{n}(x) = \sum^\infty_{j=0} \frac{1}{j! \cdot \Gamma(j + n + 1)} \left( \frac{x}{2} \right)^{2j + n}$ is the modified Bessel function.
    We note that $I_{n}(x) \sim e^x / \sqrt{2 \pi x}$ for $\x \to \infty$~\cite{bessel:apdx}.
    Therefore, 
    \begin{multline}
        \mathbb{E} \left[ \left( 1 - \frac{\lambda}{Z} \right) \cdot \mathbf{1}_{\{ Z \leq \lambda \}} \right] \\
        \geq - \frac{e^{-2B}}{B} \left( I_{\frac{d}{2}-1}(2B) + I_{\frac{d}{2}-2}(2B) \right) \in - O \left( B^{- \frac{3}{2}} \right) \enspace.
    \end{multline}
    This is the end of the proof.
\end{proof}

\subsection{Proofs in Analysis of (1+1)-LB-ES}
 \label{apdx:sec:proof:lb-es}

\begin{proof}[Proof of Proposition~\ref{proposition:drift:sig}]
    On the target problem, the successful update occurs when both of the following are satisfied.
    \begin{itemize}
        \item The generated continuous part $\x_t$ is superior to the current elitist, i.e., $\| \x_t \|^2 \leq \| \m[t] \|^2 $.
        \item The generated integer part $\z_t$ is the same as the optimal integer. i.e., $\z_t = \mathbf{0}$.
    \end{itemize}
    According to Lemma~3.1 in~\cite{akimoto:gecco:2018}, the success probability in the continuous part is bounded as 
    \begin{align}
        \Pr \left( \| \x_{t} \|^2 \leq \| \m[t] \|^2 \mid \mathcal{F}_t \right) \leq \frac{1}{2} \enspace.
        \label{eq:succ:prob:cont}
    \end{align}
    On the other hand, the success probability in the integer part is given by
    \begin{align}
        \Pr \left( \z_t = \mathbf{0} \mid \mathcal{F}_t \right) = \left( 1 - 2 \Phi\left( - \frac{1}{2 \sig[t] \langle \D[t] \rangle} \right) \right)^{\din} \enspace. 
        \label{eq:succ:prob:int}
    \end{align}
    In the following, we denote the number of successful updates between the $t$-th and $(t+s-1)$-th iterations as $N_{t,s}^{\mathrm{succ}}$.
    We note that the difference $\log (\sig[t+s]) - \log (\sig[t])$ can be written as
    \begin{align}
        \log (\sig[t+s]) - \log (\sig[t]) = \left( N_{t,s}^{\mathrm{succ}}  - \frac{s - N_{t,s}^{\mathrm{succ}}}{s-1} \right) \log \alpha \enspace.
        \label{eq:sigma:Nsucc}
    \end{align}
    We note $\sig[t] = \sig[t+s]$ when $N_{t,s}^{\mathrm{succ}} = 1$.
    We will consider two events, $N_{t,s}^{\mathrm{succ}} \geq 2$, and $N_{t,s}^{\mathrm{succ}} = 0$, separately.
    
    When $N_{t,s}^{\mathrm{succ}} \geq 2$, there is at least one successful update with the step-size satisfying $\sig[t] \langle \D[t] \rangle \geq \sigLB \cdot \alpha^{\frac{1}{1-s}}$ in the $s$ updates.
    Because RHS in~\eqref{eq:succ:prob:int} is monotonically decreasing w.r.t. $\sig[t] \langle \D[t] \rangle$, we have
    \begin{multline}
        \Pr\left( N_{t,s}^{\mathrm{succ}} \geq 2 \mid \mathcal{F}_t \right) 
        \leq s \cdot \left( 1 - 2 \Phi\left( - \frac{1}{2 \sigLB \cdot \alpha^{\frac{1}{s-1}}} \right) \right)^{\din} \enspace. 
    \end{multline}
    Because $2 \Phi\left( - {1}/{(2 \sigLB) } \right) = \new{\pinmut}$, Propositions~\ref{proposition:cdf-ratio} and~\ref{proposition:cdf-bound} show
    \begin{align*}
        2 \Phi\left( - \frac{1}{2 \sigLB \cdot \alpha^{\frac{1}{s-1}}} \right) 
        &> \new{\pinmut} \cdot \exp\left(  \gamma \cdot \left(\Phi^{-1} \left( \new{\frac{\pinmut}{2}} \right)\right)^2 \right) \\
        &\geq \new{\pinmut} \cdot \left( \sqrt{\frac{e}{2\pi}} \cdot \new{\frac{1}{\pinmut}} \right)^{\gamma} \enspace.
    \end{align*}
    Therefore, assuming $\new{\pinmut \in \Theta(1/\din)}$, we have
    \begin{align}
        \Pr\left( N_{t,s}^{\mathrm{succ}} \geq 2 \mid \mathcal{F}_t \right) 
        &\leq s \cdot \left( 1 - \new{\pinmut}  \left( \sqrt{\frac{e}{2\pi}} \cdot \new{\frac{1}{\pinmut}} \right)^{\gamma} \right)^{\din} \notag \\
        &\in O \left( \din^{- \gamma} \right) \enspace.
        \label{eq:Nsucc:2:bound}
    \end{align}

    Next, we consider the probability of the event $N_{t,s}^{\mathrm{succ}} = 0$. 
    Because of the upper bound for success probability for the continuous part in~\eqref{eq:succ:prob:cont}, we have
    \begin{align}
        \Pr\left( N_{t,s}^{\mathrm{succ}} = 0 \mid \mathcal{F}_t \right) \geq \frac{1}{2^s} \enspace.
        \label{eq:Nsucc:0:bound}
    \end{align}
    Finally, combining~\eqref{eq:sigma:Nsucc}, \eqref{eq:Nsucc:2:bound}, and \eqref{eq:Nsucc:0:bound} with $N_{t,s}^{\mathrm{succ}} \leq s$ shows the upper bound of the drift in~\eqref{eq:drift:bound:sig}.
    This is the end of the proof.
\end{proof}


\begin{proof}[Proof of Proposition~\ref{proposition:var:bound}]
    For two random variables $X,Y$, we have
    \begin{align}
        \mathrm{Var} [X - Y] &= \mathrm{Var} [X] - 2 \mathrm{Cov} [X, Y] + \mathrm{Var} [Y] \\
        | \mathrm{Cov} [X, Y] | &\leq \sqrt{ \mathrm{Var} [X] \mathrm{Var} [Y] } \enspace.
    \end{align}
    Therefore, we have
    \begin{align}
        \mathrm{Var} \left[ \log\left( \frac{\| \m[t+s] \|}{\sig[t+s]} \right) \mid \mathcal{F}_t \right] \leq \left( \sqrt{ V_{\m[],t} } + \sqrt{ V_{\sig[],t} } \right)^2 \enspace.
    \end{align}
    where $V_{\m[],t} = \mathrm{Var} \left[ \log (\| \m[t+s] \|) \mid \mathcal{F}_t \right]$ and $V_{\sig[],t} = \mathrm{Var} \left[ \log (\sig[t+s]) \mid \mathcal{F}_t \right]$.
    Because the step-size updated $s$ times is bounded as $\alpha^{- \frac{s}{s-1}} \sig[t] \leq \sig[t+s] \leq \alpha^s \sig[t]$, its logarithm is bounded as $\ell_{\sig[],t} \leq \log(\sig[t+s]) \leq u_{\sig[],t}$, where
    \begin{align}
        \ell_{\sig[],t} &:= - \frac{s}{s-1} \log(\alpha) + \log(\sig[t]) \quad \text{and} \\
        u_{\sig[],t} &:= s \log(\alpha) + \log(\sig[t]) \enspace.
    \end{align}
    Then, from Popoviciu's inequality, we obtain
    \begin{align}
        V_{\sig[],t} &\leq \frac14 \left( u_{\sig[],t} - \ell_{\sig[],t}  \right)^2 = \frac14 \left( \frac{s^{\rev{2}}}{s-1} \log (\alpha) \right)^2 \enspace.
    \end{align}

    Next, we consider the variance of $ \log( \| \m[t+s] \|)$. Because $\m[t]$ is $\mathcal{F}_t$-measurable, we can decompose the variance as
    \begin{align}
        V_{\m[],t} &= 
        \mathrm{Var} \left[ \log (\| \m[t+s] \|) \mid \mathcal{F}_t \right] \\
        &= \mathrm{Var} \left[ \log (\| \m[t+s] \|) - \log (\| \m[t] \|) \mid \mathcal{F}_t \right] \\
        &= \mathbb{E} \left[ \left( R^{t+s}_{t} \right)^2 \mid \mathcal{F}_t \right] - \left( \mathbb{E} \left[ R^{t+s}_{t} \mid \mathcal{F}_t \right] \right)^2 \enspace,
    \end{align}
    where we denote $R^{a}_{b} = \log \left( \frac{ \| \m[a] \| }{ \| \m[b] \| } \right)$ for short.
    Because the second term is negative, bounding the first term shows the claim. For $s' \in \mathbb{N}$, we have
    \begin{multline}
        \left( R^{t+s'}_{t} \right)^2 
        = \left( R^{t+s'}_{t+s'-1} \right)^2 \\
         + 2 R^{t+s'}_{t+s'-1} \cdot \underbrace{ R^{t+s'-1}_{t} }_{\text{$\mathcal{F}_{t+s'-1}$-measurable}} + \underbrace{\left( R^{t+s'-1}_{t} \right)^2}_{\text{$\mathcal{F}_{t+s'-1}$-measurable}} ~.
         \label{apdx:eq:R2}
    \end{multline}
    If $s'\geq2$, from the tower property and Proposition~\ref{proposition:drift:m}, the expectation of the second term is bounded as follows, noting $\| \m[t+s'] \| \leq \| \m[t+s'-1] \| \leq \| \m[t] \|$.
    \begin{align}
        &\E\left[ 2  R^{t+s'}_{t+s'-1} \cdot R^{t+s'-1}_{t} \mid \mathcal{F}_{t} \right] \notag \\
        &\qquad = 2\E\left[ \E\left[ R^{t+s'}_{t+s'-1} \cdot R^{t+s'-1}_{t} \!\mid\! \mathcal{F}_{t+s'-1} \right] \!\mid\!  \mathcal{F}_{t} \right] \\
        &\qquad = 2\E\left[\E\left[ R^{t+s'}_{t+s'-1} \!\mid\! \mathcal{F}_{t+s'-1} \right] R^{t+s'-1}_{t} \! \mid\! \mathcal{F}_{t} \right] \\
        &\qquad \leq 2\E\left[ - \frac{1}{\dco} \cdot R^{t+s'-1}_{t} \mid \mathcal{F}_{t} \right] \\
        &\qquad \leq \frac{2}{\dco} \cdot \frac{s'-1}{\dco} \label{eq:second_exp} \enspace.
    \end{align}
    If $s'=1$, the second term in RHS of \eqref{apdx:eq:R2} becomes $0$, which allows us to derive the upper bound of the expected value of the second term as \eqref{eq:second_exp}.
    The expectation of the first term in RHS of \eqref{apdx:eq:R2} can be obtained in the similar way as in the proof of \cite[Lemma~4.4]{akimoto:gecco:2018}.
    Let us consider the mutation direction $\boldsymbol{\delta}_t = \x_t - \m[t]$ and the optimal length $\gamma^\ast = \argmin_\gamma \| \m[t] + \gamma \boldsymbol{\delta}_t \|$ for $\boldsymbol{\delta}_t$ to produce $\x^\ast_t = \m[t] + \gamma^\ast \boldsymbol{\delta}_t$.
    Considering that $\| \m[t+1] \| \leq \| \m[t]  \|$ and $(\log(x))^2$ is monotonically decreasing for $0 < x \leq 1$, the squared log-progress is bounded as
    \begin{align}
        \left( \log (\| \m[t+1] \|) - \log (\| \m[t] \|) \right)^2 
        & \leq \left( \log (\| \x^\ast_t \|) - \log (\| \m[t] \|) \right)^2 \\
        & = \left( \log \left( \left\| \boldsymbol{e}_m + \frac{\gamma^\ast}{\| \m[t] \| } \boldsymbol{\delta}_t \right\| \right) \right)^2 \enspace,
    \end{align}
    where $\boldsymbol{e}_m = \m[t] / \| \m[t] \|$.
    The squared log-progress with the optimal length amounts to $(\log (\sin (\theta)))^2 \mathbf{1}_{\{ \theta \leq \pi / 2 \}}$, where $\theta \in [0, \pi)$ is the angle between $\boldsymbol{\delta}_t$ and $\boldsymbol{e}_m$.
    \rev{Because the mutation direction is isotropic, $\theta$ has density} $(2 W_{\dco-2})^{-1} | \sin (\theta) |^{\dco-2}$, where $W_d = \int^{\pi/2}_0 \sin^d (\theta) \mathrm{d} \theta$ denotes the Wallis integral.
    Then, the expectation of the squared log-progress is given by
    \begin{multline}
        \frac{1}{2 W_{\dco-2}} \int^{\pi /2}_0 \left( \log (\sin (\theta)) \right)^2 \sin^{\dco-2} (\theta) \mathrm{d} \theta \\
        = \left( 2 (\dco-1)^2 W_{\dco-2} \right)^{-1} \int^1_0 \frac{(\log (r))^2}{ \sqrt{1 - r^{\frac{2}{\dco-1}}} } \mathrm{d} r \enspace.
    \end{multline}
    Here we applied the change of variable $\sin^{\dco-1} (\theta) = r$. 
    When considering $r$ as a random variable generated form the uniform distribution on [0,1], $\log(r)$ and $1 / \sqrt{1 - r^{\frac{2}{\dco-1}}}$ are positively correlated~\cite[Chapter~1]{corrlation:book:2000}, which shows $(\log (r))^2$ and $1 / \sqrt{1 - r^{\frac{2}{\dco-1}}}$ are negatively correlated. Therefore, we have
    \begin{align}
        \int^1_0 \frac{(\log (r))^2}{ \sqrt{1 - r^{\frac{2}{\dco-1}}} } \mathrm{d} r 
        & \leq  \int^1_0 (\log (r))^2 \mathrm{d} r \cdot \int^1_0 \sqrt{\frac{1}{1 - r^{\frac{2}{\dco-1}}}} \mathrm{d} r \\
        & = 2 (\dco - 1) W_{\dco - 2} \enspace.
    \end{align}
    Therefore, by iteratively computing the expectation from $s'=s$ to $s'=1$, we have
    \begin{align}
        V_{\m[],t} &\leq \mathbb{E} \left[ \left( R^{t+s}_{t} \right)^2 \! \mid \mathcal{F}_{t} \right]  \\
        &\leq \sum^s_{s'=1} \left( 2 (\dco - 1) W_{\dco\! - 2} \! + \frac{2}{\dco} \cdot \frac{s'-1}{\dco} \right) \\
        &= 2 s (\dco - 1) W_{\dco - 2} + \frac{ (s-1) \cdot s}{\dco^2} \enspace.
    \end{align}
    This is the end of the proof.
    
\end{proof}

\subsection{Proofs in Analysis of (1+1)-LUB-ES}
 \label{apdx:sec:proof:lub-es}

\begin{proof}[Proof of Proposition~\ref{proposition:integer-mutation:drift}]
    Because Proposition~\ref{proposition:integer-mutation:decrease} bounds the change of $\log (\sig[t] \langle \D[t] \rangle)$, we have
    \begin{align*}
        &\max\{ V(\theta_{t+1}) - V(\theta_{t}), -A \} \\
        &\quad\leq  \max \left\{ V_{\mathrm{co}}(\theta_{t+1}) - V_{\mathrm{co}}(\theta_{t}) , -A_{\mathrm{co}} \right\} + v_{\mathrm{in}} \cdot \log \left( \frac{ \sig[t+1] \langle \D[t+1] \rangle }{ \sig[t] \langle \D[t] \rangle } \right) \enspace.
    \end{align*}
    We note that (1+1)-LB-ES generates $\x_t$ and $\z_t$ independently, and that the updates of $\m$ and $\sig$ are performed in the same way as (1+1)-ES when $\z_t = \mathbf{0}$. 
    Akimoto~et~al.~\cite[Proposition~4.2]{akimoto:gecco:2018} show an upper bound of the truncated drift for $V_{\mathrm{co}}$ with (1+1)-ES as
    \begin{align}
    \begin{split}
        \mathbb{E}_{\mathrm{ES}}&\left[ \max \left\{ V_{\mathrm{co}}(\theta_{t+1}) - V_{\mathrm{co}}(\theta_{t}) , -A_{\mathrm{co}} \right\} \mid \mathcal{F}_t \right]  \\
        &\quad \leq - \left( A_\mathrm{co} \cdot p^\ast - \frac{s\cdot v \cdot \log (\alpha)}{s-1}  \right) 
        \cdot \mathbf{1}_{ \left\{ \frac{d_\mathrm{co} \cdot \sig}{\|\m\| } \in [\ell, u] \right\} } \\
        &\qquad - \frac{s \cdot p_\ell - 1}{s-1} \cdot v \cdot \log (\alpha) 
        \cdot \mathbf{1}_{ \left\{ \frac{d_\mathrm{co} \cdot \sig}{\|\m\| } < \ell \right\} } \\
        &\qquad - \frac{1 - s \cdot p_u}{s-1} \cdot v \cdot \log (\alpha) 
        \cdot \mathbf{1}_{ \left\{ \frac{d_\mathrm{co} \cdot \sig}{\|\m\| } > u \right\} } \enspace. \label{eq:EV-cond}
    \end{split}
    \end{align}
    When $\z_t \neq \mathbf{0}$, it holds $\m[t+1] = \m$ and $\sig[t+1] = \alpha^{- \frac{1}{s-1}} \cdot \sig$ on \obj{}.
    Therefore, we have
    \begin{multline}
        \left( V_{\mathrm{co}}(\theta_{t+1}) - V_{\mathrm{co}}(\theta_{t}) \right) \cdot \mathbf{1}_{ \{ \z_t \neq \mathbf{0} \} } \\
        \leq \frac{v \cdot \log(\alpha)}{s-1} \cdot \mathbf{1}_{ \left\{ {d_\mathrm{co} \cdot \sig} / {\|\m\| } < \ell \right\} } \cdot \mathbf{1}_{ \{ \z_t \neq \mathbf{0} \} } \enspace. \label{eq:V-cond}
    \end{multline}
    \rev{
    We distinguish the two cases $\sig[t] \langle \D[t] \rangle > \sigLB$ and $\sig[t] \langle \D[t] \rangle = \sigLB$.
    If $\sig[t] \langle \D[t] \rangle > \sigLB$, the event $\z_t = \mathbf{0}$ occurs with probability $O \left( \din^{- \gamma} \right)$ by \eqref{eq:psucc:in:cases}. This yields the bound corresponding to $B_1$.
    Then, if $\sig[t] \langle \D[t] \rangle = \sigLB$, the second term of $V(\theta_t)$ is zero, so it suffices to bound the drift of $V_{\mathrm{co}}$.
    Conditioning on $\z_t = \mathbf{0}$ and $\z_t \neq \mathbf{0}$, and applying \eqref{eq:EV-cond} and \eqref{eq:V-cond}, we obtain the three subcases according to $\frac{d_\mathrm{co} \cdot \sig}{\|\m\| }$: $[\ell, u]$, $<\ell$, and $>u$, which yield the bounds $B_2$, $B_3$, and $B_4$, respectively.
    }
    Therefore, we have 
    \begin{multline*}
        \mathbb{E}\left[ \max \left\{ V(\theta_{t+1}) - V(\theta_{t}) , -A \right\} \mid \mathcal{F}_t \right] \\
        \leq - B_1 \cdot \mathbf{1}_{\{ \sig[t] \langle \D[t] \rangle > \sigLB \}}  - \min\{ B_2, B_3, B_4 \} \cdot \mathbf{1}_{\{ \sig[t] \langle \D[t] \rangle = \sigLB \}} \enspace.
    \end{multline*}
    This is the end of the proof.
\end{proof}

\section{Proof of Main Theorems} \label{apdx:proofs-main}
This section shows the proofs of our main theorems.

\begin{proof}[Proof of Theorem~\ref{theorem:premature:logm}]

    Combining Propositions~\ref{proposition:chi-squared:inverse:trunc} and~\ref{proposition:ratio:as} shows~\eqref{eq:mlog:lower}.
    Next, we show the existence of $C_\mathrm{lower} > 0$.
    Because it holds $\|\m[t+i]\| \leq \| \m[t] \| $ and $\sig[t+i] \geq \sig[t] \cdot \alpha^{- \frac{i}{s-1}}$, we have $\| \m[s \tau + i] \| / \sig[s \tau + i] \in \Theta(\| \m[s \tau] \| / \sig[s \tau]) = \Omega \left( \exp( \tau ) \right)$ for $i \in O(1)$. Then, 
    \begin{align*}
        \mathbb{E} \left[ \log \left( \frac{ \| \m[s \tau + s] \| }{ \| \m[0] \| } \right) \right] 
        & = \mathbb{E} \left[ \sum_{k=0}^\tau \sum_{j=0}^{s-1} \log \left( \frac{ \| \m[s k + j+1] \| }{ \| \m[s k + j] \| } \right) \right] \\
        & = \mathbb{E} \left[ \sum_{k=0}^\tau \sum_{j=0}^{s-1} \mathbb{E} \left[ \log \left( \frac{ \| \m[s k + j+1] \| }{ \| \m[s k + j] \| } \right) \mid \mathcal{F}_{s k + j} \right] \right] \\
        & = \sum_{k=0}^\tau s \cdot - O\left( \exp \left( - 3 k \right) \right)
        \enspace.
    \end{align*}
    Therefore, we have $\mathbb{E} \left[ \log (\| \m[s \tau + s] \| ) \right] \in - O(1) + \log( \| \m[0] \| )$.
    Finally, because $\log (\E\left[ X \right]) \geq \E\left[ \log (X) \right]$, the proof is finished.
\end{proof}

\begin{proof}[Proof of Theorem~\ref{theorem:integer-mutation:ert}]
    We first consider $B_2, B_3$, and $B_4$ one by one. Then, we have
    \begin{align}
        B_2 &= \left( A_\mathrm{co} \cdot p^\ast - \frac{s}{s-1} \cdot v \cdot \log(\alpha) \right) \cdot p^{\mathrm{succ}}_{\mathrm{in,LB}} \\
        &= \left( \frac{p^\ast}{\dco} - \frac{s}{s-1} \cdot \frac{p'}{2 \cdot \dco} \right) \cdot p^{\mathrm{succ}}_{\mathrm{in,LB}} \\
        B_3 &= \frac{s \cdot p_{\ell} \cdot p^{\mathrm{succ}}_{\mathrm{in,LB}} - 1}{s-1} \cdot v \cdot \log(\alpha) \\
        &= \frac{s \cdot p_{\ell} \cdot p^{\mathrm{succ}}_{\mathrm{in,LB}} - 1}{s-1} \cdot \frac{p'}{2 \cdot \dco} \\
        B_4 &= \frac{1 - s \cdot p_u}{s-1} \cdot v \cdot \log(\alpha) \cdot \rev{ p^{\mathrm{succ}}_{\mathrm{in,LB}}}\\
        &= \frac{1 - s \cdot p_u}{s-1} \cdot \frac{p'}{2 \cdot \dco} \cdot \rev{ p^{\mathrm{succ}}_{\mathrm{in,LB}}} \enspace.
    \end{align}
    Because $v < \frac{1}{\dco \cdot \log (\alpha)}$, it holds $r' > r^\ast$ and $p' < p^\ast$\rev{, where $r^\ast$ and $p^\ast$ are defined as in Proposition~\ref{proposition:integer-mutation:drift}}. Therefore, we have $L \leq \min\{B_2, B_3, B_4 \} \leq U$, where
    \begin{align}
        L &= \frac{p'}{\dco} \cdot c_{\min} \quad \text{and} \quad U = \frac{p^\ast}{\dco} \cdot c_{\min}
    \end{align} 
    defined with
    \begin{align*}
        c_{\min} = \min \left\{ \frac{(s - 2)\cdot p^{\mathrm{succ}}_{\mathrm{in,LB}}}{2 (s-1)}, \frac{s \cdot p_{\ell} \cdot p^{\mathrm{succ}}_{\mathrm{in,LB}} - 1}{2 (s-1)}, \frac{(1 - s \cdot p_u) \cdot \rev{ p^{\mathrm{succ}}_{\mathrm{in,LB}}}}{2 (s-1)} \right\} .
    \end{align*} 
    Since $\lim_{\dco \to \infty} \dco \cdot r^\ast = 1$ and $\lim_{\dco \to \infty} \dco \cdot r' = 1$, \citet[Lemma~3.2]{akimoto:gecco:2018} show
    \begin{align}
        \lim_{\dco \to \infty} p^\ast &= \min \left\{ \Phi \left( - \frac{1}{\ell_\mathrm{lim}} - \frac{\ell_\mathrm{lim}}{2} \right), \Phi \left( - \frac{1}{u_\mathrm{lim}} - \frac{u_\mathrm{lim}}{2} \right) \right\} > 0 \\
        \lim_{\dco \to \infty} p' &= \min_{\bar{\sigma} \in [\ell_\mathrm{lim}, u_\mathrm{lim}]} \left\{ \Phi \left( - \frac{1}{\bar{\sigma}} - \frac{\bar{\sigma}}{2} \right) \right\} > 0 \enspace,
    \end{align} 
    where the detail is found in the proof of \cite[Proposition~4.3]{akimoto:gecco:2018}, and $\ell_\mathrm{lim}$ and $u_\mathrm{lim}$ are limit values of $\ell \in \Theta(1)$ and $u \in \Theta(1)$, respectively.
    Therefore, when $s > 2 / p^{\mathrm{succ}}_{\mathrm{in,LB}}$, we have $U, L \in O(1/\dco)$ for some $0 < p_u < 1/s < p_{\ell} < 1 / 2$, which shows $\min\{B_2, B_3, B_4 \} \in O(1/\dco)$.

    Next, we rewrite $B_1$ as
    \begin{multline}
        B_1 = \frac{p'}{2 \dco \cdot (s-1)} \cdot \left( 1 - O \left( \din^{- \gamma} \right) \right) \\
        - \min\{B_2, B_3, B_4 \} \cdot O \left( \din^{- \gamma} \right) 
    \end{multline}
    Because $p'$ and $\min\{B_2, B_3, B_4 \}$ are independent from $\din$, it is obvious that there exists $D_1 \geq D_p$ such that $B_1$ satisfies $B_1 > 0$ and $B_1 \in O(1/\dco)$ for all $\din \geq D_1$.

    The second claim is shown using $\log (\| \m[T] \|) \leq V(\theta_T)$ and 
    \begin{align*}
        \mathbb{E}[ V(\theta_T) ] - V(\theta_0) 
        &= \sum^{T-1}_{i=0} \mathbb{E}\left[ V(\theta_{i+1}) - V(\theta_{i}) \right] \\
        & \leq \sum^{T-1}_{i=0} \mathbb{E}\left[ \mathbb{E}\left[ \max \{ V(\theta_{i+1}) - V(\theta_{i}), -A \} \mid \mathcal{F}_{i} \right] \right] \\
        &\leq - B \cdot T \enspace.
    \end{align*}

    Finally, we show the last claim. Applying Theorem~\ref{theorem:truncated-drift} shows an upper bound of the expected hitting time $T^V_{\epsilon} = \min\{t: V(\theta) \leq \log (\epsilon) \}$ for $V(\theta_t)$ as
    \begin{align}
        \mathbb{E}\left[ T^V_\epsilon \right] \leq \frac{V(\theta_0) - \log(\epsilon) + \frac{1}{\dco} + \frac{p'}{\dco \cdot (s-1)}}{B} \enspace,
    \end{align}
    where we rewrite $A$ in Theorem~\ref{theorem:truncated-drift} as $A = A_\mathrm{co} + \frac{v_\mathrm{in}}{s-1} \log (\alpha) = \frac{1}{\dco} + \frac{p'}{\dco \cdot (s-1)}$.
    Because $\log( \| \m \|) \leq V(\theta_t)$, we have $T_\epsilon \leq T^V_\epsilon$. 
    Since $\Phi^{-1}(1/x) \sim - \sqrt{2 \log(x)} $ as $x \to \infty$~\cite{dominici:arxiv:2003}, we have $\sigLB \in \Theta \bigl( ( \log( \new{ 1/\pinmut } )^{- \frac{1}{2}} \bigr) $, which shows 
    \begin{align}
        V(\theta_0) \in \Theta( V_\mathrm{co}(\theta_0) + \log ( \log( \new{ 1/\pinmut } ) ) ) \enspace.
    \end{align}
    Assuming $\pinmut \in O(1/\dco)$, we obtain $V(\theta_0) \in \Theta( V_\mathrm{co}(\theta_0) + \linebreak \log ( \log( \dco ) ) )$.
    Finally, the lower bound follows by upper-bounding the one-step log-progress by the corresponding hit-and-run progress, as in \cite[Lemma~4.4]{akimoto:gecco:2018}, and applying \cite[Theorem~2.3]{akimoto:gecco:2018}.
    Therefore, $\mathbb{E}\left[ T_\epsilon \right] \in \Omega(\dco \cdot \log(1/\epsilon))$.
    This is the end of the proof.
\end{proof}



\end{document}